\documentclass{article}

\PassOptionsToPackage{numbers, compress,sort}{natbib}

\usepackage[preprint]{neurips_2021}

\usepackage[utf8]{inputenc} 
\usepackage[T1]{fontenc}    
\usepackage{hyperref}       
\usepackage{url}            
\usepackage{booktabs}       
\usepackage{amsfonts}       
\usepackage{nicefrac}       
\usepackage{microtype}      
\usepackage{xcolor}         

\usepackage{amsfonts}
\usepackage{amssymb,amsthm,amsmath}
\usepackage{esint}
\usepackage{isomath}
\usepackage{latexsym}
\usepackage{epic}
\usepackage{epsfig}
\usepackage{multirow}
\usepackage{hhline}
\usepackage{wrapfig}
\usepackage{verbatim}
\usepackage{enumitem}
\usepackage{color}
\usepackage{thmtools}
\usepackage{thm-restate}
\usepackage[ruled,vlined, linesnumbered, boxed]{algorithm2e}

\SetCommentSty{mycommfont}
\usepackage{mathtools}
\usepackage{graphicx}
\usepackage{tabu}
\usepackage{colortbl}

\usepackage[labelformat=simple]{subcaption}

\newcommand{\R}{\mathbb{R}}

\newtheorem{definition}{Definition}

\newcommand{\feature}{x}
\newcommand{\Feature}{\mathcal{X}}
\newcommand{\model}{m} 
\newcommand{\parameter}{\theta}
\newcommand{\Parameter}{\Theta}
\newcommand{\weight}{w}
\newcommand{\policy}{\pi}
\newcommand{\trajectories}{\mathcal{T}}
\newcommand{\trajectory}{\tau}
\newcommand{\dataset}{\mathcal{D}}

\title{Learning MDPs from Features: Predict-Then-Optimize for Sequential Decision Problems by Reinforcement Learning
}

\author{%
  Kai Wang \\
  Harvard University\\
  Cambridge, MA\\
  \texttt{kaiwang@g.harvard.edu} \\
  \And
  Sanket Shah \\
  Harvard University \\
  Cambridge, MA\\
  \texttt{sanketshah@g.harvard.edu} \\
  \And
  Haipeng Chen \\
  Harvard University \\
  Cambridge, MA\\
  \texttt{hpchen@seas.harvard.edu} \\
  \AND
  Andrew Perrault\thanks{This work was done while the author was at Harvard University.} \\
  The Ohio State University\\
  Columbus, OH\\
  \texttt{perrault.17@osu.edu } \\
  \And
  Finale Doshi-Velez \\
  Harvard University \\
  Cambridge, MA\\
  \texttt{finale@seas.harvard.edu} \\
  \And
  Milind Tambe \\
  Harvard University \\
  Cambridge, MA\\
  \texttt{milind\_tambe@harvard.edu} \\
}

\begin{document}

\maketitle

\begin{abstract}
In the predict-then-optimize framework, the objective is to train a predictive model, mapping from environment features to parameters of an optimization problem, which maximizes decision quality when the optimization is subsequently solved. Recent work on decision-focused learning shows that embedding the optimization problem in the training pipeline can improve decision quality and help generalize better to unseen tasks compared to relying on an intermediate loss function for evaluating prediction quality. We study the predict-then-optimize framework in the context of \emph{sequential} decision problems (formulated as MDPs) that are solved via reinforcement learning. In particular, we are given environment features and a set of trajectories from training MDPs, which we use to train a predictive model that generalizes to unseen test MDPs without trajectories. Two significant computational challenges arise in applying decision-focused learning to MDPs: (i) large state and action spaces make it infeasible for existing techniques to differentiate through MDP problems, and (ii) the high-dimensional policy space, as parameterized by a neural network, makes differentiating through a policy expensive. We resolve the first challenge by sampling provably unbiased derivatives to approximate and differentiate through optimality conditions, and the second challenge by using a low-rank approximation to the high-dimensional sample-based derivatives. We implement both Bellman--based and policy gradient--based decision-focused learning on three different MDP problems with missing parameters, and show that decision-focused learning performs better in generalization to unseen tasks.

\end{abstract}

\section{Introduction}
\emph{Predict-then-optimize}~\cite{elmachtoub2021smart,balghiti2019generalization} is a framework for solving optimization problems with unknown parameters. Given such a problem, we first train a predictive model to predict the missing parameters from problem features. Our objective is to maximize the resulting decision quality when the optimization problem is subsequently solved with the predicted parameters~\cite{sahinidis2004optimization,ning2019optimization}.
Recent work on the \emph{decision-focused learning} approach~\cite{donti2017task,wilder2019melding} embeds the optimization problem~\cite{amos2018differentiable,agrawal2019differentiable,bai2019deep} into the training pipeline and trains the predictive model end-to-end to optimize the final decision quality.
Compared with a more traditional ``two-stage'' approach which maximizes the predictive accuracy of the model (rather than the final decision quality), the decision-focused learning approach can achieve a higher solution quality and generalize better to unseen tasks.

This paper studies the predict-then-optimize framework in \emph{sequential} decision problems, formulated as Markov decision processes (MDPs), with unknown parameters. In particular, at training time, we are given trajectories and environment features from ``training MDPs.''
Our goal is to learn a predictive model which maps from environment features to missing parameters based on these trajectories that generalizes to unseen test MDPs that have features, but not trajectories.
The resulting ``predicted'' training and test MDPs are solved using deep reinforcement learning (RL) algorithms, yielding policies that are then evaluated by offline off-policy evaluation (OPE) as shown in Figure~\ref{fig:flowchart}. This fully offline setting is motivated by real-world applications such as wildlife conservation and tuberculosis treatment where no simulator is available.
However, such domains offer past ranger patrol trajectories and environmental features of individual locations from conservation parks for generalization to other unpatrolled areas. These settings differ from those considered in transfer-RL~\cite{ng1999policy,taylor2009transfer,lazaric2012transfer,schaul2015universal} and meta-RL~\cite{wang2016learning,duan2016rl,finn2017model,zintgraf2019varibad,wang2020offline} because we generalize across different MDPs by explicitly predicting the mapping function from features to missing MDPs parameters, while transfer/meta RL achieve generalization by learning hidden representation of different MDPs implicitly with trajectories. 



The main contribution of this paper is to extend the decision-focused learning approach to MDPs with unknown parameters, embedding the MDP problems in the predictive model training pipeline.
To perform this embedding, we study two common types of optimality conditions in MDPs: a Bellman-based approach where mean-squared Bellman error is minimized, and a policy gradient-based approach where the expected cumulative reward is maximized.
We convert these optimality conditions into their corresponding Karush--Kuhn--Tucker (KKT) conditions, where we can backpropagate through the embedding by differentiating through the KKT conditions.
However, existing techniques from decision-focused learning and differentiating through KKT conditions do not directly apply as the size of the KKT conditions of sequential decision problems grow linearly in the number of states and actions, which are often combinatorial or continuous and thus become intractable.

We identify and resolve two computational challenges in applying decision-focused learning to MDPs, that come up in both optimality conditions: (i) the large state and action spaces involved in the optimization reformulation make differentiating through the optimality conditions intractable and (ii) the high-dimensional policy space in MDPs, as parameterized by a neural network, makes differentiating through a policy expensive.
To resolve the first challenge, we propose to sample an estimate of the first-order and second-order derivatives to approximate the optimality and KKT conditions. We prove such a sampling approach is unbiased for both types of optimality conditions. Thus, we can differentiate through the approximate KKT conditions formed by sample-based derivatives.
Nonetheless, the second challenge still applies---the sampled KKT conditions are expensive to differentiate through due to the dimensionality of the policy space when model-free deep RL methods are used. Therefore, we propose to use a low-rank approximation to further approximate the sample-based second-order derivatives. This low-rank approximation reduces both the computation cost and the memory usage of differentiating through KKT conditions.

We empirically test our decision-focused algorithms on three settings:
a grid world with unknown rewards, and snare-finding and Tuberculosis treatment problems where transition probabilities are unknown. 
Decision-focused learning  achieves better OPE performance in unseen test MDPs than two-stage approach, and our 
low-rank approximations significantly scale-up decision-focused learning.


\section*{Related Work}
\paragraph{Differentiable optimization}~\label{sec:related-work-differentiable-optimization}
Amos et al.~\cite{amos2017optnet} propose using a quadratic program as a differentiable layer and embedding it into deep learning pipeline, and Agrawal et al.~\cite{agrawal2019differentiable} extend their work to convex programs.
Decision-focused learning~\cite{donti2017task,wilder2019melding} focuses on the predict-then-optimize~\cite{elmachtoub2021smart,balghiti2019generalization} framework by embedding an optimization layer into training pipeline, where the optimization layers can be convex~\cite{donti2017task}, linear~\cite{wilder2019melding,mandi2020interior}, and non-convex~\cite{perrault2020end,wang2020automatically}.
Unfortunately, these techniques are of limited utility for sequential decision problems because their formulations grow linearly in the number of states and actions and thus differentiating through them quickly becomes infeasible.
Amos et al.~\cite{amos2018differentiable} avoid this issue by studying model-predictive control but limited to quadratic-form actions, reducing the dimensionality.
Karkus et al.~\cite{karkus2019differentiable} differentiate through an algorithm by unrolling and relaxing all the strict operators by soft operators.
Futoma et al.~\cite{futoma2020popcorn} deal with large optimality conditions by differentiating through the last step of the value-iteration algorithm only.
Instead, our approach does not rely on any MDP solver structure.
We combine sampling and a low-rank approximation to form an unbiased estimate of the optimality conditions to differentiate through, and show that the approach of Futoma et al.~\cite{futoma2020popcorn} is included in ours as a special case.

\paragraph{Predict-then-optimize and offline reinforcement learning}
The idea of planning under a predicted MDP arises in model-based RL as \textit{certainty equivalence}~\cite{kumar2015stochastic}. It has been  extended to offline settings~\cite{kidambi2020morel,yu2021combo}, who learn a pessimistic MDP before solving for the policy.
Our setting differs because of the presence of features and train-test split---our test MDPs are completely fresh \textit{without any trajectories}. Our setting also resembles meta RL (e.g., \cite{wang2016learning,duan2016rl,finn2017model,zintgraf2019varibad,wang2020offline}) and transfer RL (e.g., \cite{ng1999policy,taylor2009transfer,lazaric2012transfer,schaul2015universal}.)
Meta RL focuses on training a ``meta policy'' for a distribution of tasks (MDPs), leveraging trajectories for each. Transfer RL works by extracting transferable knowledge from source MDPs to target MDPs using trajectories. In contrast to these two paradigms, ours explicitly trains a predictive model (which maps problem features to missing MDP parameters) to generalize knowledge learned from the training set to the testing set using \textit{problem features, not trajectories}.

\section{Problem Statement}\label{sec:problem}

\begin{figure}
    \centering
    \includegraphics[width=0.98\linewidth]{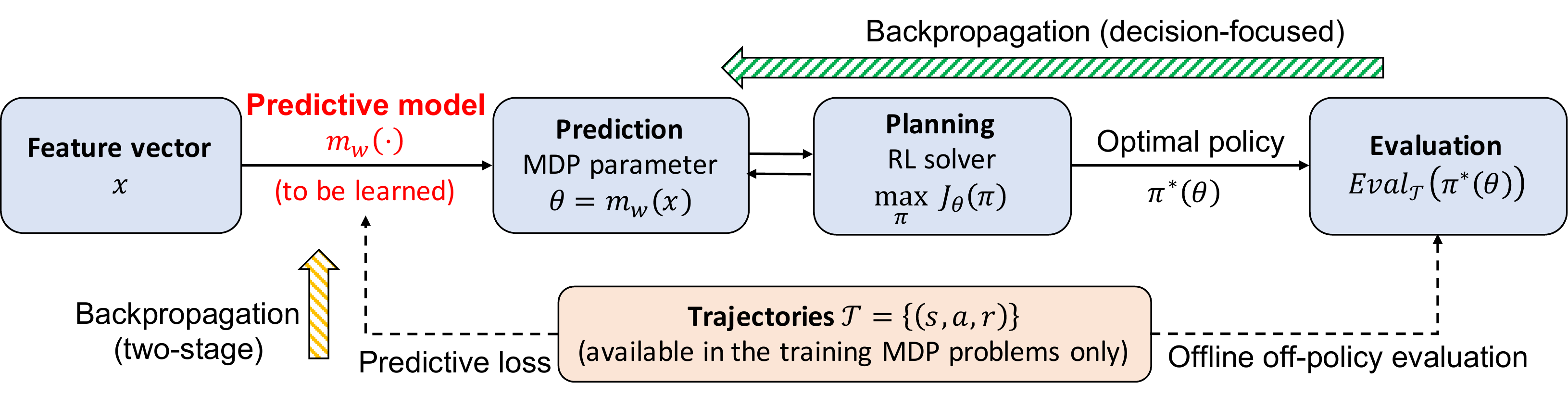}
    \caption{We consider learning a predictive model to map from features to unknown MDP parameters and obtaining a policy by solving the predicted MDP with RL. Two-stage learning learns the predictive model by minimizing a predictive loss function, whereas decision-focused learning is trained end-to-end to maximize the final off-policy evaluation performance.}
    \label{fig:flowchart}
\end{figure}

In this paper, we consider learning a predictive model to infer the missing parameters in a sequential decision-making task (formulated as MDPs) using the predict-then-optimize framework.
Each MDP is defined by a tuple $(\mathcal{S}, \boldsymbol{s}_0, \mathcal{A}, T, R)$ with an initial state $\boldsymbol{s}_0$, a possibly infinite state space $\mathcal{S}$ and possibly infinite action space $\mathcal{A}$.
We assume some parameters are missing in each MDP, which could be any portion of the transition function $T$ and the reward function $R$. We denote the missing parameters vector by $\parameter^*$.
Additionally, we assume there are problem features $\feature$ associated with each MDP, where $(\theta^*, \feature)$ is correlated and drawn from the same unknown, but fixed, distribution\footnote{Examples of the missing parameters $\parameter^*$ include the poaching risk of different locations in wildlife conservation and the transition probability of patients' healthiness in healthcare problems, where the corresponding problem features are terrain features of different locations and the characteristics of different patients that are correlated to the missing parameters, respectively. These correlated features allow us to predict the missing parameters even if we do not have any trajectories of the MDP.}.

We are given a set of training MDPs and a set of test MDPs, each with missing parameters $\theta^*$ and features $\feature$. 
Each training MDP is accompanied by a set of trajectories $\trajectories$ performed by a behavior policy $\policy_{\text{beh}}$, consisting of a sequence of states, actions and rewards. In the test MDPs, trajectories from the behavior policy are hidden at test time. These testing MDPs are considered fresh instances that we have to generate a policy without using any trajectories.
Thus, at training time, we learn a predictive model $\model_\weight$ to map from features to missing parameters; at test time, we apply the same model to make predictions and plan accordingly without using trajectories.

Formally, our goal is to learn a predictive model $\model_\weight$ to predict the missing parameters $\parameter = \model_\weight(\feature)$.
The predicted parameters are used to solve the test MDPs, yielding the policy $\policy^*(\model_\weight(\feature))$ where our offline off-policy evaluation (OPE) metric is maximized. This framework is illustrated in Figure~\ref{fig:flowchart}. 

\paragraph{Offline off-policy evaluation} We evaluate a policy $\policy$ in a fully offline setting with trajectories $\trajectories = \{ \trajectory_i \}, \trajectory_i = (s_{i1}, a_{i1}, r_{i1}, \cdots, s_{ih}, a_{ih}, r_{ih})$ generated from the MDP using behavior policy $\policy_{\text{beh}}$. We use an OPE metric used by Futoma et~al.~\cite{futoma2020popcorn} --- we evaluate a policy $\policy$ and trajectories $\trajectories$ as:
\begin{align}\label{eqn:ope}
    \text{Eval}_\trajectories(\policy) & \coloneqq V^{\text{CWPDIS}}(\policy) - \lambda_{\text{ESS}} / \sqrt{\text{ESS}(\policy)}
\end{align}
where $V^{\text{CWPDIS}}(\policy) \! \coloneqq \! \sum\nolimits_{t=1}^h \gamma^t \frac{\sum\nolimits_{i} r_{it }\rho_{it}(\policy) }{\sum\nolimits_{i} \rho_{it}(\policy)}$ and $\text{ESS}(\policy) \! \coloneqq \! \sum\nolimits_{t=1}^h \frac{(\sum\nolimits_{i} \rho_{it})^2}{\sum\nolimits_{i} \rho_{it}^2}$, and $\rho_{it}(\policy)$ is the ratio of the proposed policy and the behavior policy likelihoods up to time t: $\rho_{it}(\policy) \coloneqq \prod\nolimits_{t'=1}^t \frac{\policy(a_{it'}|s_{it'})}{\policy_{\text{beh}}(a_{it'}|s_{it'})}$.

\paragraph{Optimization formulation}
Given a set of training features and trajectories $D_\text{train} = \{(\feature_i, \trajectories_i)\}_{i \in I_{\text{train}}}$, our goal is to learn a predictive model $\model_\weight$ to optimize the training performance:
\begin{align}\label{eqn:optimization-problem}
    \max\nolimits_\weight \quad \mathop{\mathbb{E}}\nolimits_{(\feature, \trajectories) \in D_{\text{train}}} \left[ \text{Eval}_\trajectories(\policy^*(\model_\weight(\feature))) \right]
\end{align}
The testing performance is evaluated on the unseen test set $D_\text{test} = \{(\feature_i, \trajectories_i)\}_{i \in I_{\text{test}}}$ with trajectories hidden from training, and only used for evaluation: $\mathop{\mathbb{E}}\nolimits_{(\feature, \trajectories) \in D_{\text{test}}} \left[ \text{Eval}_\trajectories(\policy^*(\model_\weight(\feature))) \right]$.


\section{Two-stage Learning}\label{sec:two-stage}
To learn the predictive model $\model_\weight$ from trajectories, the standard approach is to minimize an expected predictive loss, which is defined by comparing the prediction $\parameter = \model_\weight(\feature)$ with the trajectories $\trajectories$:
\begin{align}\label{eqn:two-stage-loss}
    & \min\nolimits_\weight \quad \mathop{\mathbb{E}}\nolimits_{(\feature, \trajectories) \sim \mathcal{D}_{\text{train}}} \mathcal{L}(\parameter, \trajectories)
    & \text{where} \quad \mathcal{L}(\parameter, \trajectories) = \mathop{\mathbb{E}}\nolimits_{\trajectory \sim \trajectories} \ell_\parameter(\trajectory), \quad \parameter = \model_\weight(\feature)
\end{align}
For example, when the rewards are missing, the loss could be the squared error between the predicted rewards and the actual rewards we see in the trajectories for each individual step. When the transition probabilities are missing, the loss could be defined as the negative log-likelihood of seeing the trajectories in the training set.

In the first stage, to train the predictive model, we run gradient descent to minimize the loss function defined in Equation~\eqref{eqn:two-stage-loss} and make prediction on the model parameter $\parameter = \model_\weight(\feature)$ of each problem. In the second stage, we solve each MDP problem with the predicted parameter $\parameter$ using an RL algorithm to generate the optimal policy $\policy^*(\parameter)$. However, predictive loss and the final evaluation metric are commonly misaligned especially in deep learning problems with imbalanced data~\cite{huang2019addressing,lambert2020objective,johnson2019survey,branco2016survey}.
This motivates us to learn the predictive model end-to-end and therefore avoid the misalignment.

\section{Decision-focused Learning}\label{sec:decision-focused}
In this section, we present our main contribution, decision-focused learning in sequential decision problems, as illustrated in Figure~\ref{fig:flowchart}.
Decision-focused learning integrates an MDP problem into the training pipeline to directly optimize the final performance.
Instead of relying on a predictive loss to train the predictive model $\model_\weight$, we can directly optimize the objective in Equation~\eqref{eqn:optimization-problem} by running end-to-end gradient descent to update the predictive model $\model_\weight$:
\begin{align}\label{eqn:decision-focused-gradient}
    \frac{d ~ \text{Eval}(\policy^*)}{d \weight} = \frac{d ~ \text{Eval}(\policy^*)}{d \policy^*} \frac{d \policy^*}{d \parameter} \frac{d \parameter}{d \weight}
\end{align}
We assume the policy $\policy^*(\parameter)$ is either stochastic and smooth with respect to the change in the parameter $\parameter$, which is common in settings with continuous state or action spaces, or that an appropriate regularization term~\cite{haarnoja2017reinforcement,haarnoja2018soft} is used to improve the smoothness of the policy. More discussions about the smoothness can be found in Appendix~\ref{sec:soft-rl-solver}.

This gradient computation requires us to back-propagate from the final evaluation through the MDP layer to the predictive model $\model_\weight$ that we want to update.
The major challenge in Equation~\eqref{eqn:decision-focused-gradient} is to compute $\frac{d \policy^*}{d \parameter}$, which involves differentiating through an MDP layer solved by an RL algorithm. In the following section, we first discuss two different optimality conditions in MDPs, which are later used to convert into KKT conditions and differentiate through to compute $\frac{d \policy^*}{d \parameter}$. We then discuss two computational challenges associated with the derivative computation.

\subsection{Optimality Conditions in MDPs}\label{sec:optimality-conditions}
When the predicted model parameter $\parameter = \model_\weight(\feature)$ is given, the MDP can be solved by any RL algorithm to get an optimal policy $\policy^*$.
Here we discuss two common optimality conditions in MDPs, differing by the use of policy gradient or Bellman equation:
\begin{definition}[Policy gradient-based optimality condition]\label{def:policy-gradient-optimality-condition}
Defining $J_\parameter(\policy)$ to be the expected cumulative reward under policy $\policy$, the optimality condition of the optimal policy $\policy^*$ is:
\begin{align}\label{eqn:policy-gradient-optimality-condition}
    & \policy^* = \arg\max\nolimits_{\policy} J_{\parameter}(\policy) & \text{where} \quad J_{\parameter}(\policy) \coloneqq \mathop{\mathbb{E}}\nolimits_{\tau \sim \policy, \parameter} G_\parameter(\trajectory)
\end{align}
where $G_\parameter(\trajectory)$ is the discounted value of trajectory $\trajectory$ given parameter $\parameter$, and the expectation is taken over the trajectories following the optimal policy and transition probability (as part of $\parameter$).
\end{definition}

\begin{definition}[Bellman-based optimality condition]\label{def:bellman-optimality-condition}
Defining $J_\parameter(\policy)$ to be the mean-squared Bellman error\footnote{We use the same notation $J$ to denote both the expected cumulative reward and the expected Bellman error to simplify the later analysis of decision-focused learning.} under policy $\policy$, the optimality condition of the optimal policy $\policy^*$ (valuation function) is:
\begin{align}\label{eqn:bellman-optimality-condition}
    & \policy^* = \arg\min_\policy J_\parameter(\policy) & \text{where} \quad J_\parameter(\policy) \coloneqq \mathop{\mathbb{E}}\nolimits_{\trajectory \sim \policy, \parameter} \frac{1}{2} \delta_\parameter^2(\trajectory, \policy)
\end{align}
where $\delta_\parameter(\trajectory, \policy) = \sum\nolimits_{(s,a,s') \in \trajectory} Q_\policy(s,a) - R_\parameter(s,a) - \gamma \mathop{\mathbb{E}}\nolimits_{a' \sim \policy} Q_\policy(s', a')$ is the total Bellman error of a trajectory $\trajectory$, and each individual term $\delta_\parameter(\trajectory, \policy)$ can depend on the parameter $\parameter$ because the Bellman error depends on the immediate reward $R_\parameter$, which can be a part of the MDP parameter $\parameter$. The expectation in Equation~\eqref{eqn:bellman-optimality-condition} is taken over all the trajectories generated from policy $\policy$ and transition probability (as part of $\parameter$).
\end{definition}



\subsection{Backpropagating Through Optimality and KKT Conditions}
To compute the derivative of the optimal policy $\policy^*(\parameter)$ in an MDP with respect to the MDP parameter $\parameter$, we differentiate through the KKT conditions of the corresponding optimality conditions:
\begin{definition}[KKT Conditions]
Given objective $J_\parameter(\policy)$ in an MDP problem, since the policy parameters are unconstrained, the necessary KKT conditions can be written as: $\nabla_\policy J_{\parameter}(\policy^*) = 0$.
\end{definition}

In particular, computing the total derivative of KKT conditions gives:
\begin{align}
    0 &= \frac{d}{d \parameter} \nabla_\policy J_{\parameter}(\policy^*) \! 
    = \! \frac{\partial}{\partial \parameter} \nabla_{\policy} J_{\parameter}(\policy^*) + \frac{\partial}{\partial \policy} \nabla_{\policy} J_{\parameter}(\policy^*) \frac{d \policy^*}{d \parameter} \!
    = \! \nabla^2_{\parameter \policy} J_{\parameter}(\policy^*) \! + \! \nabla^2_{\policy} J_{\parameter}(\policy^*) \frac{d \policy^*}{d \parameter} \nonumber \\
    & \Longrightarrow \frac{d \policy^*}{d \parameter} = -(\nabla^2_{\policy} J_{\parameter}(\policy^*))^{-1} \nabla^2_{\parameter \policy} J_{\parameter}(\policy^*) \label{eqn:implicit-function-theorem}
\end{align}
We can use Equation~\eqref{eqn:implicit-function-theorem} to replace the term $\frac{d \policy^*}{d \parameter}$ in Equation~\eqref{eqn:decision-focused-gradient} to compute the full gradient to back-propagate from the final evaluation to the predictive model parameterized by $\weight$:
\begin{align}\label{eqn:full-gradient}
    \frac{d ~ \text{Eval}(\policy^*)}{d \weight} = - \frac{d ~ \text{Eval}(\policy^*)}{d \policy^*} (\nabla^2_{\policy} J_{\parameter}(\policy^*))^{-1} \nabla^2_{\parameter \policy} J_{\parameter}(\policy^*) \frac{d \parameter}{d \weight}
\end{align}

\subsection{Computational Challenges in Backward Pass}
Unfortunately, although we can write down and differentiate through the KKT conditions analytically, we cannot explicitly compute the second-order derivatives $\nabla^2_{\policy} J_{\parameter}(\policy^*)$ and $\nabla^2_{\parameter\policy} J_{\parameter}(\policy^*)$ in Equation~\eqref{eqn:full-gradient} due to the following two challenges:
    
\paragraph{Large state and action spaces involved in optimality conditions}
The objectives $J_\parameter(\policy^*)$ in Definition~\ref{def:policy-gradient-optimality-condition} and Definition~\ref{def:bellman-optimality-condition} involve an expectation over all possible trajectories, which is essentially an integral and is intractable when either the state or action space is continuous. This prevents us from explicitly verifying optimality and writing down the two derivatives $\nabla^2_\policy J_\parameter(\policy^*)$ and $\nabla^2_{\parameter \policy} J_\parameter(\policy^*)$.

\paragraph{High-dimensional policy space parameterized by neural networks}
In MDPs solved by model-free deep RL algorithms, the policy space $\policy \in \Pi$ is often parameterized by a neural network, which has a significantly larger number of variables than standard optimization problems. This large dimensionality makes the second-order derivative $\nabla^2_\policy J_\parameter(\policy^*) \in \R^{\dim({\policy}) \times \dim({\policy})}$ intractable to compute, store, or invert.

\section{Sampling Unbiased Derivative Estimates}\label{sec:sampling-approach}
In both policy gradient--based and Bellman--based optimality conditions, the objective is implicitly given by an expectation over all possible trajectories, which could be infinitely large when either state or action space is continuous.
This same issue arises when expressing such an MDP as a linear program --- there are infinitely many constraints, making it intractable to differentiate through.

Inspired by the policy gradient theorem, although we cannot compute the exact gradient of the objective, we can sample a set of trajectories $\trajectory = \{ s_1, a_1,r_1, \dots, s_h, a_h, r_h \}$ from policy $\policy$ and model parameter $\parameter$ with finite time horizon $h$. Denoting $p_\parameter(\trajectory, \policy)$ to be the likelihood of seeing trajectory $\trajectory$, we can compute an unbiased derivative estimate for both optimality conditions:
\begin{restatable}[Policy gradient-based unbiased derivative estimate]{theorem}{policyGradientUnbiased}\label{thm:policy-gradient-unbiased}
We follow the notation of Definition~\ref{def:policy-gradient-optimality-condition} and define $\Phi_\parameter(\trajectory, \policy) = \sum\nolimits_{i=1}^h \sum\nolimits_{j=i}^h \gamma^j R_\parameter(s_j, a_j) \log \policy(a_i | s_i)$. We have:
\begin{align}\label{eqn:policy-gradient-sampling}
        \nabla_\policy J_\parameter(\policy) = \mathop{\mathbb{E}}\nolimits_{\trajectory \sim \policy, \parameter} \left[ \nabla_\policy \Phi_\parameter(\trajectory, \policy) \right] \ \Longrightarrow \
        \begin{aligned}
            \nabla^2_{\policy} J_\parameter(\policy) &=  \mathop{\mathbb{E}}\nolimits_{\trajectory \sim \policy, \parameter} \left[ \nabla_\policy \Phi_\parameter \cdot \nabla_\policy \log p_{\theta}^\top + \nabla^2_\policy \Phi_\parameter \right] \\
            \nabla^2_{\parameter \policy} J_\parameter(\policy) &= \mathop{\mathbb{E}}\nolimits_{\trajectory \sim \policy, \parameter} \left[ \nabla_\policy \Phi_\parameter \cdot \nabla_\parameter \log p_{\theta}^\top + \nabla^2_{\parameter \policy} \Phi_\parameter \right]
        \end{aligned}
\end{align}
\end{restatable}

\begin{restatable}[Bellman-based unbiased derivative estimate]{theorem}{bellmanUnbiased}\label{thm:bellman-unbiased}
We follow the notation in Definition~\ref{def:bellman-optimality-condition} to define $J_\parameter(\policy) = \frac{1}{2} \mathop{\mathbb{E}}\nolimits_{\trajectory \sim \policy, \parameter} \left[ \delta_\parameter^2(\trajectory, \policy) \right]$. We have:
\begin{align}\label{eqn:bellman-sampling}
& \nabla_\policy J_\parameter(\policy) = \mathop{\mathbb{E}}\nolimits_{\trajectory \sim \policy, \parameter} \left[ \delta \nabla_\policy \delta + \frac{1}{2} \delta^2 \nabla_\policy \log p_\parameter \right]
\ \Longrightarrow \ \nabla^2_{\policy} J_\parameter(\policy) =  \mathop{\mathbb{E}}\nolimits_{\trajectory \sim \policy, \parameter} \left[ \nabla_\policy \delta \nabla_\policy \delta^\top + O(\delta) \right] \nonumber \\
& \nabla^2_{\parameter \policy} J_\parameter(\policy) \! = \! \mathop{\mathbb{E}}\nolimits_{\trajectory \sim \policy, \parameter} \left[ \nabla_\policy \delta \nabla_\parameter \delta^\top \! + \! \left( \nabla_\policy \delta \nabla_\parameter \log p_\parameter^\top \! + \! \nabla_\policy \log p_\parameter \nabla_\parameter \delta^\top \! + \! \nabla^2_{\parameter \policy} \delta \right) \delta + O(\delta^2) \right]
\end{align}
\end{restatable}

For the latter, we apply the fact that at the near-optimal policy, the Bellman error is close to $0$ and thus each individual component $\delta(\trajectory, \policy)$ is small to simplify the analysis. Refer to the appendix for the full derivations of Equations~\eqref{eqn:policy-gradient-sampling} and~\eqref{eqn:bellman-sampling}.

Equations~\eqref{eqn:policy-gradient-sampling} and~\eqref{eqn:bellman-sampling} provide a sampling approach to compute the second-order derivatives, avoiding computing an expectation over the large trajectory space.
We can use the optimal policy derived in the forward pass and the predicted parameters $\parameter$ to run multiple simulations to collect a set of trajectories. These trajectories from predicted parameters can be used to compute each individual derivative in Equations~\eqref{eqn:policy-gradient-sampling} and~\eqref{eqn:bellman-sampling}.


\section{Resolving High-dimensional Derivatives by Low-rank Approximation}\label{sec:hessian}
Section~\ref{sec:sampling-approach} provides sampling approaches to compute an unbiased estimate of second-order derivatives.
However, since the dimensionality of the policy space $\dim(\policy)$ is often large, we cannot explicitly expand and invert $\nabla^2_{\policy} J_{\parameter}(\policy^*)$ to compute $\nabla^2_{\policy} J_{\parameter}(\policy^*)^{-1} \nabla^2_{\parameter\policy} J_{\parameter}(\policy^*)$, which is an inevitable step toward computing the full gradient of decision-focused learning in Equation~\eqref{eqn:full-gradient}.
In this section, we discuss various ways to approximate $\nabla^2_{\policy} J_{\parameter}(\policy^*)$ and how we use low-rank approximation and Woodbury matrix identity~\cite{woodbury1950inverting} to efficiently invert the sampled Hessian without expanding the matrices. Let $n \coloneqq \dim(\policy)$ and $k \ll n$ to be the number of trajectories we sample to compute the derivatives.

\subsection{Full Hessian Computation}\label{sec:full-hessian-computation}
In Equations~\eqref{eqn:policy-gradient-sampling} and~\eqref{eqn:bellman-sampling}, we can apply auto-differentiation tools to compute all individual derivatives in the expectation. However, this works only when the dimensionality of the policy space $\policy \in \Pi$ is small because the full expressions in Equations~\eqref{eqn:policy-gradient-sampling} and~\eqref{eqn:bellman-sampling} involve computing second-order derivatives , e.g., $\nabla^2_\policy \Phi_\parameter$ in Equation~\eqref{eqn:bellman-sampling}, which is still challenging to compute and store when the matrix size $n \times n$ is large. The computation cost is $O(n^2 k) + O(n^\omega)$ dominated by computing all the Hessian matrices and the matrix inversion with $2 \leq \omega \leq 2.373$ the complexity order of matrix inversion.

\subsection{Approximating Hessian by Constant Identity Matrix}\label{sec:identity-hessian}
One naive way to approximate the Hessian $\nabla^2_{\policy} J_{\parameter}(\policy^*)$ is to simply use a constant identity matrix $c I$. We choose $c < 0$ for the policy gradient--based optimality in Definition~\ref{def:policy-gradient-optimality-condition} because the optimization problem is a maximization problem and thus is locally concave at the optimal solution, whose Hessian is negative semi-definite. Similarly, we choose $c > 0$ for the Bellman--based optimality in Definition~\ref{def:bellman-optimality-condition}. This approach is fast, easily invertible. Moreover, in the Bellman version, Equation~\eqref{eqn:full-gradient} is equivalent to the idea of differentiating through the final gradient of Bellman error as proposed by Futoma et al.~\cite{futoma2020popcorn}\footnote{The gradient of Bellman error can be written as $\nabla_{\pi} J_{\theta}(\pi^*)$ where the policy $\pi$ is the parameters of the value function approximator and $J$ is defined as the expected Bellman error. The derivative of the final gradient can be written as $\nabla_{\weight} (\nabla_{\pi} J_{\theta}(\pi^*)) = \nabla^2_{\theta \pi} J_{\theta}(\pi^*) \frac{d \theta}{d w}$ by chain rule, which matches the last three terms in Equation~\ref{eqn:full-gradient} when the Hessian is approximated by an identity matrix.}. However, this approach ignores the information provided by the Hessian term, which can often lead to instability as we later show in the experiments. In this case, the computation complexity is dominated by computing $\nabla^2_{\parameter\policy} J_\parameter(\policy^*)$, which requires $O(nk)$.

\subsection{Low-rank Hessian Approximation and Application of Woodbury Matrix Identity}\label{sec:woodbury}
A compromise between the full Hessian and using a constant matrix is approximating the second-order derivative terms in Equations~\eqref{eqn:policy-gradient-sampling} and~\eqref{eqn:bellman-sampling} by constant identity matrices, while computing the first-order derivative terms with auto-differentiation.
Specifically, given a set of $k$ sampled trajectories $\{ \trajectory_1, \trajectory_2, \cdots, \trajectory_k \}$, Equations~\eqref{eqn:policy-gradient-sampling} and~\eqref{eqn:bellman-sampling} can be written and approximated in the following form:
\begin{align}\label{eqn:low-rank-approximation}
    \nabla^2_{\policy} J_\parameter(\policy) \approx \frac{1}{k} \sum\nolimits_{i=1}^k \left( u_i v_i^\top + H_i \right) \approx \frac{1}{k} \sum\nolimits_{i=1}^k \left( u_i v_i^\top + c I \right) = UV^\top + cI
\end{align}
where $U = [u_1, u_2, \cdots, u_k] / \sqrt{k} \in \R^{n \times k}, V = [v_1, v_2, \cdots, v_k] / \sqrt{k} \in \R^{n \times k}$ and $u_i, v_i \in \R^{n}$ correspond to the first-order derivatives in Equations~\eqref{eqn:policy-gradient-sampling} and~\eqref{eqn:bellman-sampling}, and $H_i$ corresponds to the remaining terms that involve second-order derivatives.
We use a constant identity matrix to approximate $H_i$, while explicitly computing the remaining parts to increase accuracy.

However, we still cannot explicitly expand $U V^\top \in \R^{n \times n}$ since the dimensionality is too large.
Therefore, we apply Woodbury matrix identity~\cite{woodbury1950inverting} to invert Equation~\eqref{eqn:low-rank-approximation}:
\begin{align}
    (\nabla^2_{\policy} J_\parameter(\policy))^{-1} \approx (U V^\top + cI)^{-1} =  \frac{1}{c} I - \frac{1}{c} U(cI - V^\top U)^{-1} V^\top
\end{align}
where $V^\top U \in \R^{k \times k}$ can be efficiently computed with much smaller $k \ll n$. When we compute the full gradient for decision-focused learning in Equation~\eqref{eqn:full-gradient}, we can then apply matrix-vector multiplication without expanding the full high-dimensional matrix, which results in a computation cost of $O(nk + k^\omega)$ that is much smaller than the full computation cost $O(n^2k +n^\omega)$.

The full algorithm for decision-focused learning in MDPs is illustrated in Algorithm~\ref{alg:decision-focused-learning}.

\begin{algorithm}[t]
\caption{Decision-focused Learning for MDP Problems with Missing Parameters}
\label{alg:decision-focused-learning}
\textbf{Parameter:} Training set $\dataset_{\text{train}} = \{ (\feature_i, \trajectories_i) \}_{i \in I_\text{train}}$, learning rate $\alpha$, regularization $\lambda = 0.1$ \\
\textbf{Initialization:} Initialize predictive model $\model_\weight: \Feature \rightarrow \Parameter$ parameterized by $\weight$ \\
\For{epoch $\in \{1, 2, \dots \}$, each training instance $(\feature, \trajectories) \in \dataset_{\text{train}}$} {
        \textbf{Forward:} Make prediction $\theta = \model_\weight(\feature)$. Compute two-stage loss $\mathcal{L}(\parameter, \trajectories)$. Run model-free RL to get an optimal policy $\policy^*(\parameter)$ on MDP problem using parameter $\parameter$. \\
        \textbf{Backward:} Sample a set of trajectories under $\parameter, \policy^*$ to compute $\nabla^2_{\policy} J_\parameter(\policy^*), \nabla^2_{\theta \policy} J_\parameter(\policy^*)$ \\
        \textbf{Gradient step:} Set $\Delta \weight = \frac{d ~ \text{Eval}_\trajectories(\policy^*)}{d \weight} - \lambda \frac{d \mathcal{L}(\parameter, \trajectories)}{d \weight}$ by Equation~\eqref{eqn:full-gradient} with predictive loss $\mathcal{L}$ as regularization. Run gradient ascent to update model: $\weight \gets \weight + \alpha \Delta \weight$ \\
}
\textbf{Return:} Predictive model $\model_\weight$.
\end{algorithm}


\section{Example MDP Problems with Missing Parameters}\label{sec:experiment_setting}

\paragraph{Gridworld with unknown cliffs}
We consider a Gridworld environment with a set of training and test MDPs. Each MDP is a $5 \! \times \! 5$ grid with a start state located at the bottom left corner and a safe state with reward drawn from $\mathcal{N}(5, 1)$ located at the top right corner. Each intermediate state has a reward associated with it, where most of them give the agent a reward drawn from $\mathcal{N}(0, 1)$ but $20\%$ of the them are cliffs and give $\mathcal{N}(-10, 1)$ penalty to the agent.
The agent has no prior information about the reward of each grid cell (i.e., the reward functions of the MDPs are unknown), but has a feature vector per grid cell correlated to the reward, and a set of historical trajectories from the training MDPs.
The agent learns a predictive model to map from the features of a grid cell to its missing reward information, and the resulting MDP is used to plan.
Since the state and action spaces are both finite, we use tabular value-iteration~\cite{sutton1990integrated} to solve the MDPs. 

\paragraph{Partially observable snare-finding problems with missing transition function}
We consider a set of {\it synthetic} conservation parks, each with $20$ sites, that are vulnerable to poaching activities.
Each site in a conservation park starts from a \textit{safe} state and has an unknown associated probability that a poacher places a snare at each time step. 
Motivated by~\cite{xu2020stay}, we assume a ranger who can visit one site per time step and observes whether a snare is present. If a snare is present, the ranger removes it and receives reward $1$. Otherwise, the ranger receives reward of $-1$.
The snare can stay in the site if the ranger does not remove it, which makes the snare-finding problem a sequential problem rather than a multi-armed bandit problem.
As the ranger receives no information about the sites that they do not visit, the MDP belief state is the ranger's belief about whether a snare is present.
The ranger uses the features of each site and historical trajectories to learn a predictive model of the missing transition probability of a snare being placed. 
Since the belief state is continuous and the action space is discrete, given a predictive model of the missing transition probability, the agent uses double deep Q-learning (DDQN)~\cite{van2016deep} to solve the predicted MDPs.

\paragraph{Partially observable patient treatment problems with missing transition probability}
We consider a version of the Tuberculosis Adherence Problem explored in~\cite{mate2020collapsing}.
Given that the treatment for tuberculosis requires patients to take medications for an extended period of time, one way to improve patient adherence is Directly Observed Therapy, in which a healthcare worker routinely checks in on the patient to ensure that they are taking their medications.
In our problem, we consider 5 {\it synthetic} patients who have to take medication for 30 days. Each day, a healthcare worker chooses one patient to intervene on. They observe whether that patient is adhering or not, and improve the patient's likelihood of adhering on that day, 
where we use the number of adherent patients as the reward to the healthcare worker.
Whether a patient actually adheres or not is determined by a transition matrix that is randomly drawn from a fixed distribution inspired by~\cite{killian2019learning}.
The aim of the prediction stage is to use the features associated with each patient, e.g., patient characteristics, to predict the missing transition matrices.
The aim of the RL stage is then to create an intervention strategy for the healthcare worker such that the sum of patient adherence over the 30-day period is maximised. Due to partial observability, we convert the problem to its continuous belief state equivalence and solve it using DDQN. 

Please refer to Appendix~\ref{sec:experiment-setup} for more details about problem setup in all three domains.
\definecolor{lightred}{HTML}{FA8C84}
\definecolor{lightgreen}{HTML}{A0E384}

\begin{table}[t]
\caption{OPE performances of different methods on the test MDPs averaged over $30$ independent runs.
Decision-focused learning methods consistently outperform two-stage approach, with some exception using identity matrix based Hessian approximation which may lead to high gradient variance.
}\label{table:ope-trajectories}
\vspace{2mm}
\small
\centering
\begin{tabular}{lcccccc}
\toprule
& \multicolumn{2}{c}{Gridworld} & \multicolumn{2}{c}{Snare} & \multicolumn{2}{c}{Tuberculosis} \\ \cmidrule(lr{.75em}){2-3} \cmidrule(lr{.75em}){4-5} \cmidrule(lr{.75em}){6-7}
Trajectories        & Random                  & Near-optimal          & Random                 & Near-optimal          & Random                   & Near-optimal  \\ \midrule
TS                  & $ \cellcolor{lightred} -12.0 \! \pm \! 1.3 $ & $ \cellcolor{lightred} 4.2 \! \pm \! 0.8 $ & $ 0.8 \! \pm \! 0.3 $  & $ 3.7 \! \pm \! 0.3 $ & $ \cellcolor{lightred} 35.8 \! \pm \! 1.5 $ & $\cellcolor{lightred} 38.7 \! \pm \! 1.6$ \\
PG-Id               & $ -11.7 \! \pm \! 1.2 $ & $ \cellcolor{lightgreen} 5.7 \! \pm \! 0.8 $ & $ \cellcolor{lightred} -0.1 \! \pm \! 0.3 $ & $ \cellcolor{lightred} 3.6 \! \pm \! 0.3 $ & $38.4 \! \pm \! 1.5$ & $40.7 \! \pm \! 1.7$ \\
Bellman-Id          & $ \cellcolor{lightgreen} -9.6 \! \pm \! 1.4 $  & $ 4.6 \! \pm \! 0.7 $ & $ 0.7 \! \pm \! 0.4 $  & $ \cellcolor{lightred} 3.6 \! \pm \! 0.3 $ & $\cellcolor{lightgreen} 39.1 \! \pm \! 1.7$ & $40.8 \! \pm \! 1.7$\\
PG-W                & $ -11.2 \! \pm \! 1.2 $ & $ 5.5 \! \pm \! 0.8 $ & $ 1.2 \! \pm \! 0.4 $  & $ \cellcolor{lightgreen} 4.8 \! \pm \! 0.3 $ & $38.4 \! \pm \! 1.5$ & $40.8 \! \pm \! 1.7$ \\
Bellman-W           & $ -11.3 \! \pm \! 1.4 $ & $ 4.8 \! \pm \! 0.8 $ & $ \cellcolor{lightgreen} 1.5 \! \pm \! 0.4 $  & $ 4.3 \! \pm \! 0.3 $ & $38.6 \! \pm \! 1.6$ & $\cellcolor{lightgreen} 41.1 \! \pm \! 1.7$ \\ \bottomrule
\end{tabular}
\end{table}

\section{Experimental Results and Discussion}
In our experiments, we compare two-stage learning ({\bf TS}) with different versions of decision-focused learning ({\bf DF}) using two different optimality conditions, policy gradient ({\bf PG}) and Bellman equation-based ({\bf Bellman}), and two different Hessian approximations ({\bf Identity}, {\bf Woodbury}) defined in Section~\ref{sec:hessian}. Computing the {\bf full} Hessian (as in Section~\ref{sec:full-hessian-computation}) is computationally intractable.
Across all three examples, we use $7$ training MDPs, $1$ validation MDP, and $2$ test MDPs, each with $100$ trajectories.
The predictive model is trained on the training MDP trajectories for $100$ epochs.
Performance is evaluated under the Off-Policy Evaluation (OPE) metric of Equation~\eqref{eqn:ope} with respect to the withheld test trajectories. In the following, we will discuss \textit{how} DF variants work compared with TS methods, and explore \textit{why} some methods are better.
We use two different trajectories, {\bf random} and {\bf near-optimal}, in the training MDP to model different imbalanced information given to train the predictive model. The results are shown in Table~\ref{table:ope-trajectories}.

\paragraph{Decision-focused learning with the Woodbury matrix identity outperforms two-stage learning} Table~\ref{table:ope-trajectories} summarizes the average OPE performance on the test MDPs. We can see that in all of the three problem settings, the best performances are all achieved by decision-focused learning. However, when Hessian approximation is not sufficiently accurate, decision-focused learning can sometimes perform even worse than two-stage (e.g., PG-Id and Bellman-Id in the snare problem). In contrast, decision-focused methods using a more accurate low-rank approximation and Woodbury matrix identity (i.e., PG-W and Bellman-W), as discussed in Section~\ref{sec:woodbury}, dominate two-stage performance in the test MDPs across all settings.

\paragraph{Low predictive loss does not imply a winning policy}
In Figures~\ref{fig:gridworld-training-loss},~\ref{fig:snare-training-loss}, we plot the predictive loss curve in the training MDPs over different training epochs of Gridworld and snare problems.
In particular, two-stage approach is trained to minimize such loss, but fails to win in Table~\ref{table:ope-trajectories}. Indeed, low predictive loss on the training MDPs does not always imply a high off-policy evaluation on the training MDPs in Figure~\ref{fig:gridworld-training-eval} due to the misalignment of predictive accuracy and decision quality, which is consistent with other studies in mismatch of predictive loss and evaluation metric~\cite{huang2019addressing,lambert2020objective,johnson2019survey,branco2016survey}.

\paragraph{Comparison between different Hessian approximations}
In Table~\ref{table:ope-trajectories}, we notice that more inaccurate Hessian approximation (identity) does not always lead to poorer performance.
We hypothesize that this is due to the non-convex off-policy evaluation objective that we are optimizing, where higher variance might sometimes help escape local optimum more easily.
The identity approximation is more unstable across different tasks and different trajectories given. In Table~\ref{table:ope-trajectories}, the performance of Bellman-Identity and PG-Identity sometimes lead to wins over two stage and sometimes losses.

\begin{figure*}[t]
\centering
\begin{minipage}{0.66\textwidth}
    \centering
    \begin{subfigure}{\linewidth}
        \includegraphics[width=\textwidth]{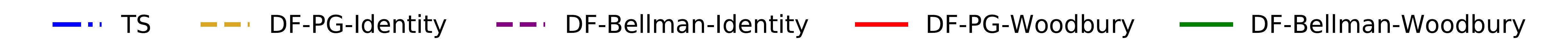}
    \end{subfigure}
    \begin{subfigure}{0.49\linewidth}
        \centering
        \includegraphics[width=\textwidth,height=2.45cm]{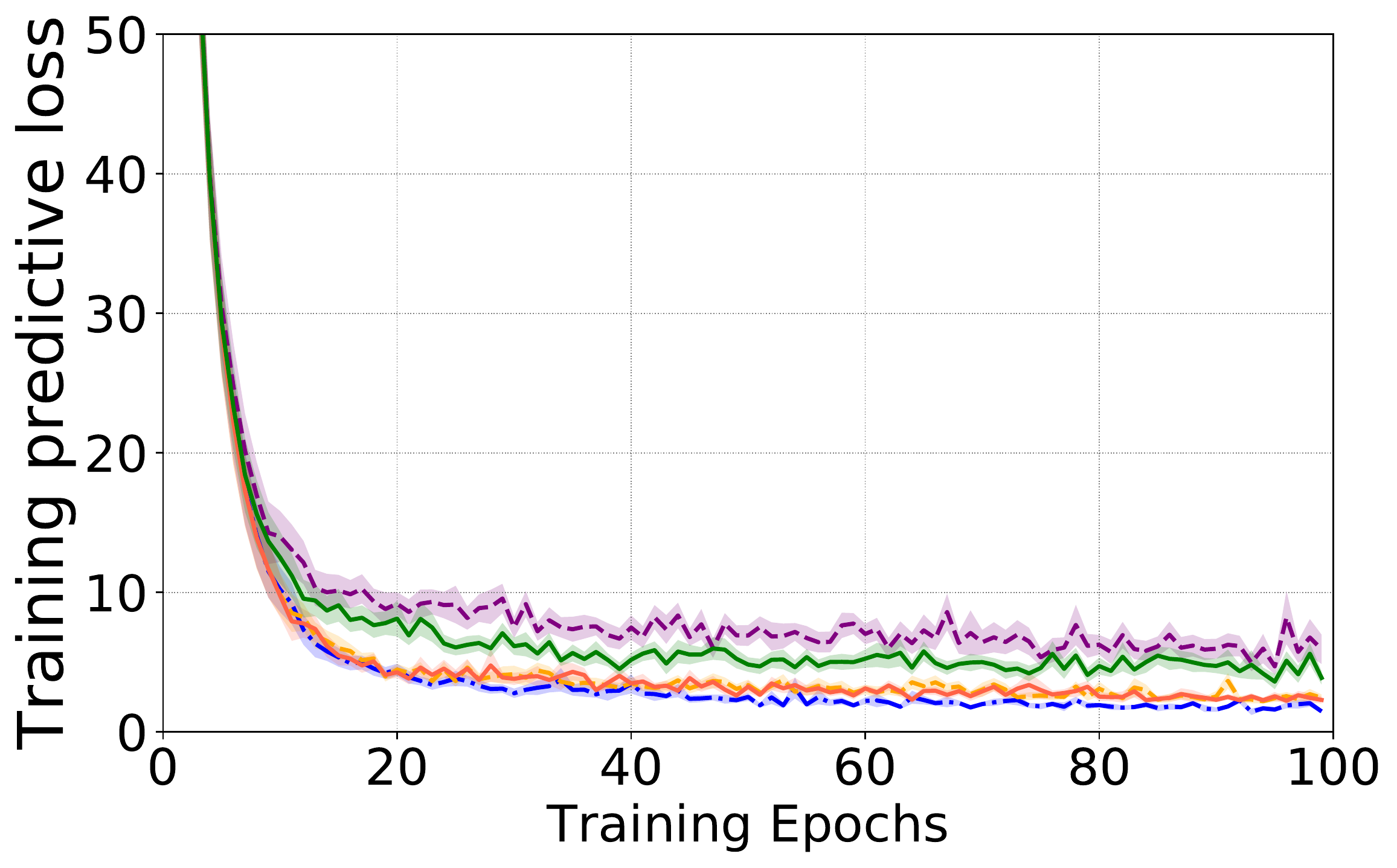}
        \caption{Predictive loss in training MDPs}
        \label{fig:gridworld-training-loss}
    \end{subfigure}
    \hfill
    \begin{subfigure}{0.49\linewidth}
        \centering
        \includegraphics[width=\textwidth,height=2.45cm]{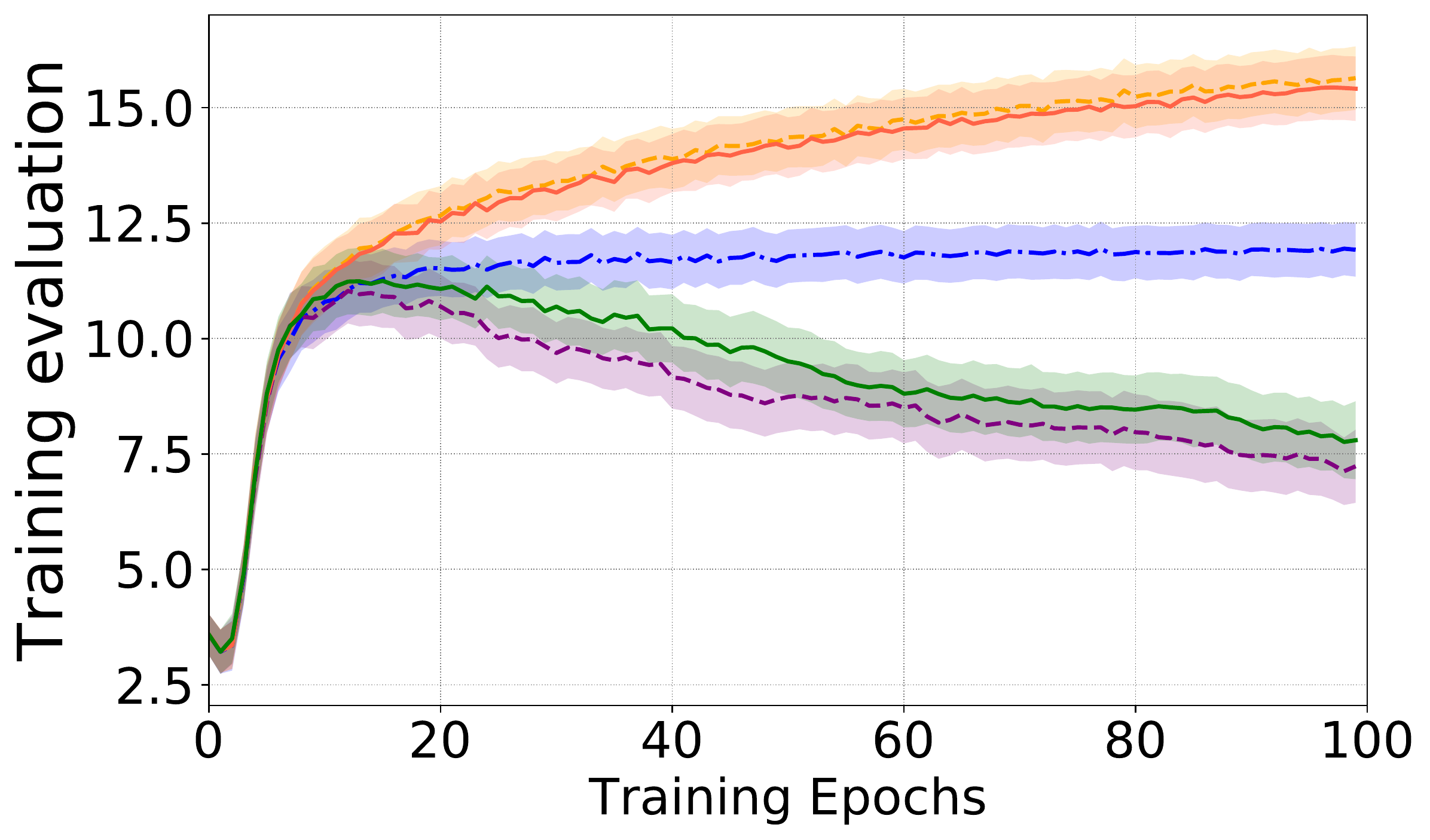}
        \caption{Performance in training MDPs}
        \label{fig:gridworld-training-eval}
    \end{subfigure}
\end{minipage}
\hfill
\begin{minipage}{0.33\textwidth}
    \includegraphics[width=\linewidth,height=2.45cm]{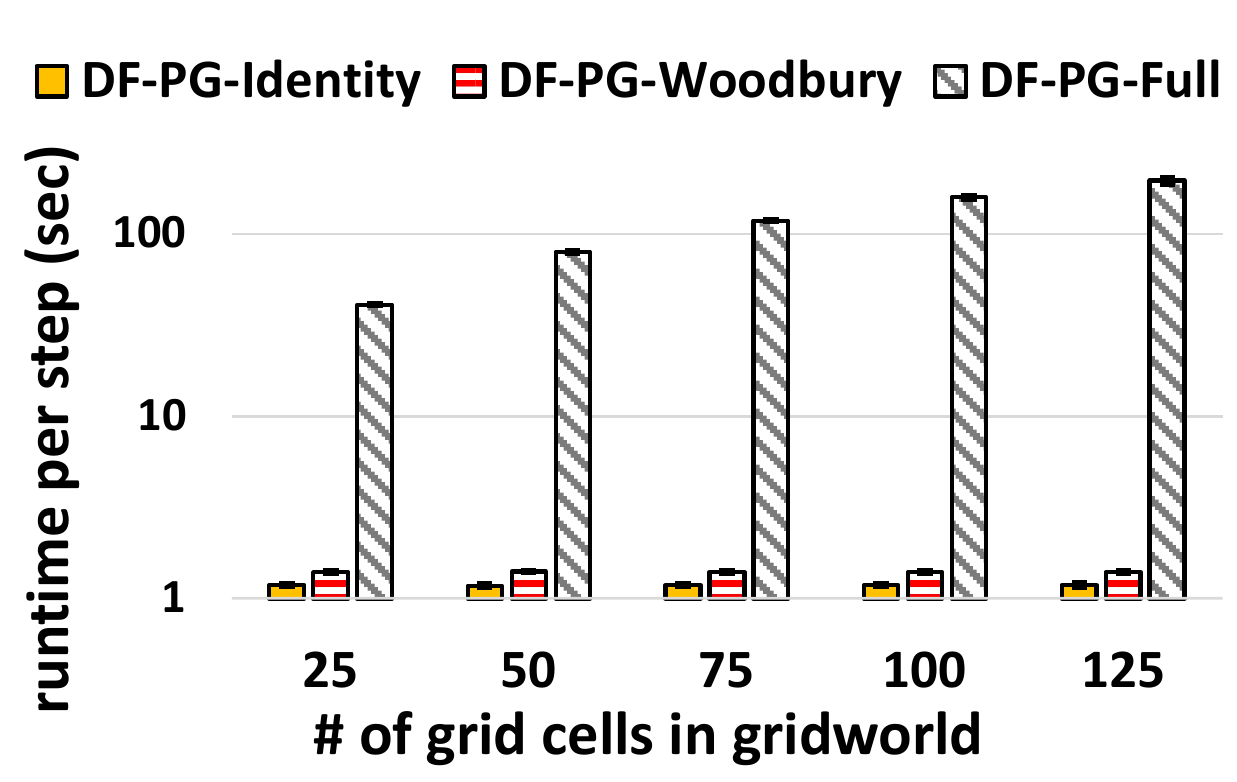}
    \subcaption{Backpropagation runtime per gradient step per instance}
    \label{fig:gridworld-runtime}
\end{minipage}
    \caption{Learning curves of Gridworld problem with near-optimal trajectories. Two-stage minimizes the predictive loss in Figure~\ref{fig:gridworld-training-loss}, but this does not lead to good training performance in Figure~\ref{fig:gridworld-training-eval}.  Figure~\ref{fig:snare-runtime} shows the backpropagation runtime per gradient step per instance of three Hessian approximations, which becomes intractable when trained for multiple instances and multiple epochs.}
    \label{fig:gridworld}
\end{figure*}
\begin{figure*}[t]
\centering
\begin{minipage}{0.66\textwidth}
    \centering
    \begin{subfigure}{\linewidth}
        \includegraphics[width=\textwidth]{figures/legend_short.pdf}
    \end{subfigure}
    \begin{subfigure}{0.49\linewidth}
        \centering
        \includegraphics[width=\textwidth,height=2.45cm]{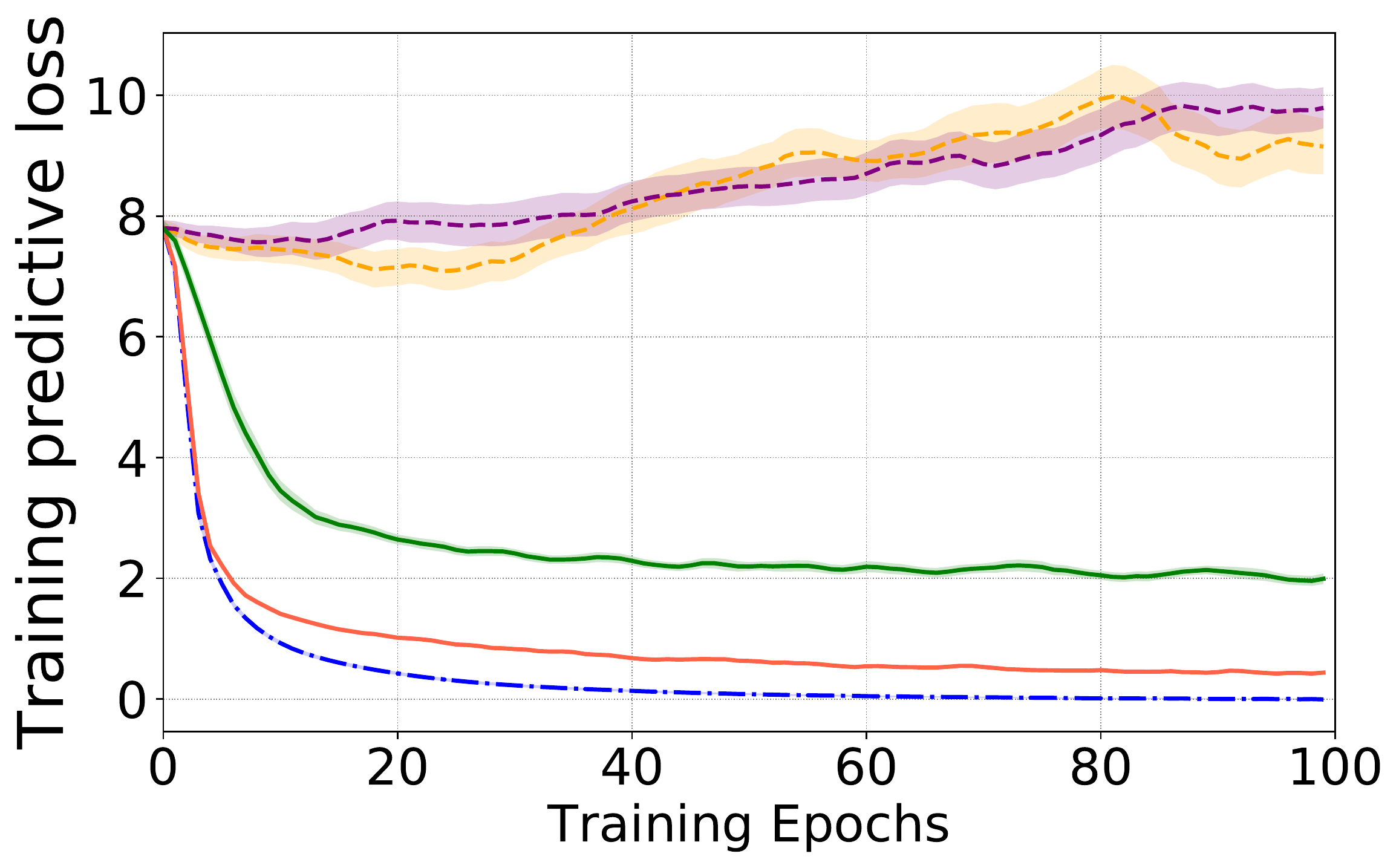}
        \caption{Predictive loss in training MDPs}
        \label{fig:snare-training-loss}
    \end{subfigure}
    \hfill
    \begin{subfigure}{0.49\linewidth}
        \centering
        \includegraphics[width=\textwidth,height=2.45cm]{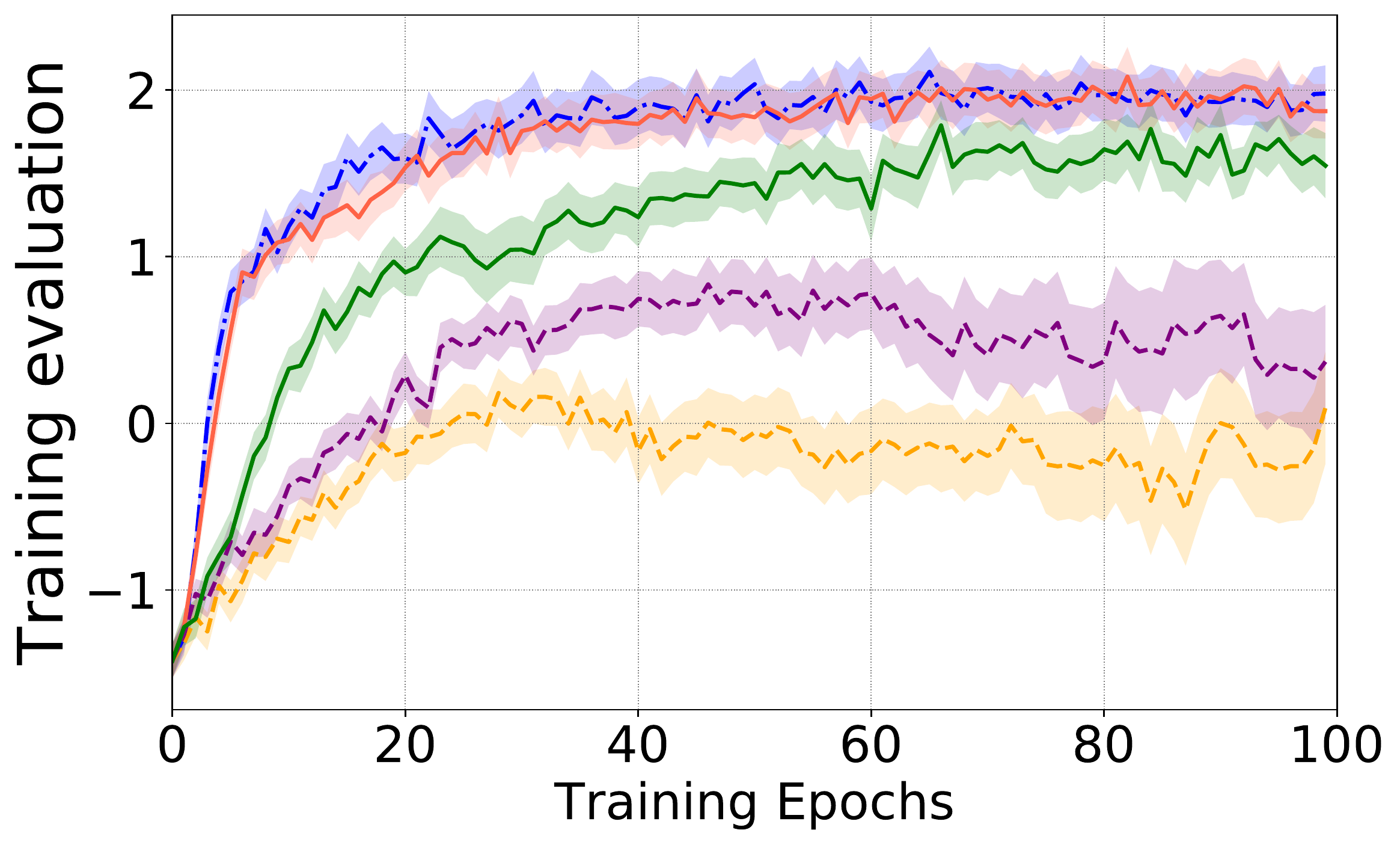}
        \caption{Performance in training MDPs}
        \label{fig:snare-training-eval}
    \end{subfigure}
\end{minipage}
\hfill
\begin{minipage}{0.33\textwidth}
    \includegraphics[width=\linewidth,height=2.45cm]{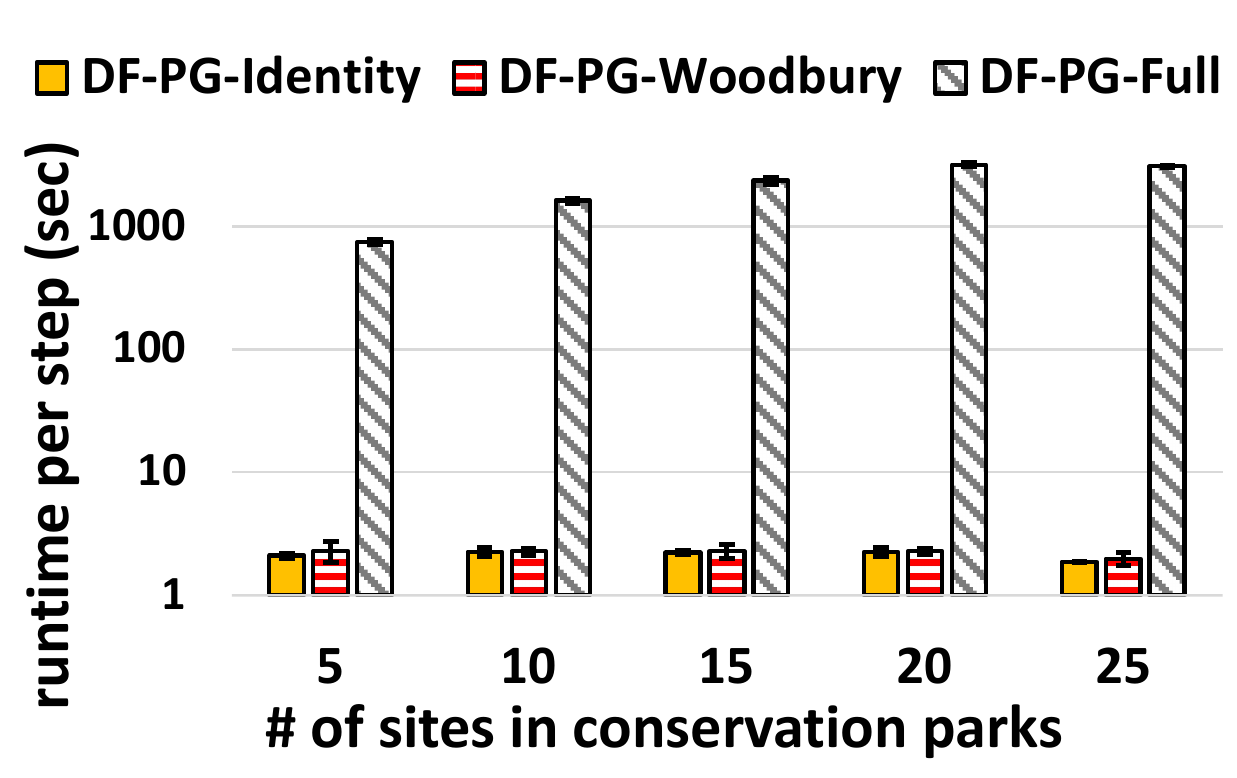}
    \subcaption{Backpropagation runtime per gradient step per instance}
    \label{fig:snare-runtime}
\end{minipage}
    \caption{Learning curves of snare finding problems with random trajectories. Two-stage achieves both low predictive loss in Figure~\ref{fig:snare-training-loss} and high training OPE in Figure~\ref{fig:snare-training-eval}, but the test performance is poor in Table~\ref{table:ope-trajectories}.
    Figure~\ref{fig:snare-runtime} plots the backpropagation runtime per gradient step per instance.}
    \label{fig:snare}
\end{figure*}

\paragraph{Comparison between policy gradient and Bellman-based decision-focused learning}
We observe that the Bellman-based decision-focused approach consistently outperforms the policy gradient-based approach when the trajectories are random, while the policy gradient-based decision-focused approach mostly achieves better performance when near-optimal trajectories are present. We hypothesize that this is due to the different objectives considered by different optimality conditions. The Bellman error aims to accurately cover \emph{all} the value functions, which works better on random trajectories; the policy gradient aims to maximize the expected cumulative reward along the \emph{optimal policy only}, which works better with near-optimal trajectories that have better coverage in the optimal regions.

\paragraph{Computation cost}
Lastly, Figures~\ref{fig:gridworld-runtime} and~\ref{fig:snare-runtime} show the backpropagation runtime of the policy-gradient based optimality condition per gradient step per training instance across different Hessian approximations and different problem sizes in the gridworld and snare finding problems. To train the model, we run $100$ epochs for every MDP in the training set, which immediately makes the full Hessian computation intractable as it would take more than a day to complete.

Analytically, let $n$ be the dimensionality of the policy space and $k \ll n$ be the number of sampled trajectories used to approximate the derivatives. As shown in Section~\ref{sec:hessian}, the computation cost of full Hessian $O(n^2 + n^\omega)$ is quadratic in $n$ and strictly dominates all the others.
In contrast, the costs of the identity matrix approximation $O(nk)$ and the Woodbury approximation $O(nk + k^\omega)$ are both linear in $n$. The Woodbury method offers an option to get a more accurate Hessian at low additional cost.


\section{Conclusion}
This paper considers learning a predictive model to address the missing parameters in sequential decision problems.
We successfully extend decision-focused learning from optimization problems to MDP problems solved by deep reinforcement learning algorithms, where we apply sampling and low-rank approximation to Hessian matrix computation to address the associated computational challenges.
All our results suggest that decision-focused learning can outperform two-stage approach by directly optimizing the final evaluation metric.
The idea of considering sequential decision problems as differentiable layers also suggests a different way to solve online reinforcement learning problems, which we reserve as a future direction.

\section*{Acknowledgement}
This research was supported by MURI Grant Number W911NF-17-1-0370 and W911NF-18-1-0208. Chen and Perrault acknowledge support from the Center for Research on Computation and Society.
Doshi-Velez acknowledges support from NSF CAREER IIS-1750358.
The computations in this paper were run on the FASRC Cannon cluster supported by the FAS Division of Science Research Computing Group at Harvard University.

\section*{Declaration of Interests}
Doshi-Velez reports grants from the National Institutes of Health, and personal fees from Davita Kidney Care and Google Health via Adecco, outside of the submitted work.
Tambe is jointly affiliated with Google Research India, outside of the submitted work. Wang, Shah, Chen, and Perrault declare no competing interests.

\bibliographystyle{abbrv}
\bibliography{reference}

\newpage
\appendix
\section*{Appendix}

\section{Proofs}\label{sec:missing-proofs}

\policyGradientUnbiased*
\begin{proof}[First part of the proof (policy gradient theorem)]

The first part of the proof follows the policy gradient theorem. We begin with definitions.

Let $\trajectory = \{s_1, a_1, s_2, a_2, \cdots, s_h, a_h \} $ be a trajectory sampled according to policy $\policy$ and MDP parameter $\parameter$. Define $\trajectory_j = \{s_1, a_1, \cdots, s_{j}, a_{j}\}$ to be a partial trajectory up to time step $j$ for any $j \in [h]$.
Define $G_\parameter(\trajectory) = \sum\nolimits_{j=1}^{h} \gamma^j R_\parameter(s_j, a_j)$ to be the discounted value of trajectory $\trajectory$.
Let $p_\parameter(\trajectory, \policy)$ be the probability of seeing trajectory $\trajectory$ under parameter $\parameter$ and policy $\policy$.
Given MDP parameter $\parameter$, we can compute the expected cumulative reward of policy $\policy$ by:
\begin{align}
    J_\parameter(\policy) &= \mathop{\mathbf{E}}\nolimits_{\trajectory \sim \policy, \parameter} G_\parameter(\trajectory) \nonumber \\
    &= \mathop{\mathbf{E}}\nolimits_{\trajectory \sim \policy, \parameter} \sum\nolimits_{j=1}^h \gamma^j R_\parameter(s_j, a_j) \nonumber \\
    &= \sum\nolimits_{j=1}^h \mathop{\mathbf{E}}\nolimits_{\trajectory \sim p_\parameter(\trajectory,\policy)} \gamma^i R_\parameter(s_j, a_j) \label{eqn:policy-gradient-full-expectation} \\
    &= \sum\nolimits_{j=1}^h \mathop{\mathbf{E}}\nolimits_{\trajectory_j \sim p_\parameter(\trajectory_j,\policy)} \gamma^j R_\parameter(s_j, a_j) \label{eqn:policy-gradient-partial-expectation} \\
    &= \sum\nolimits_{j=1}^h \int\nolimits_{\trajectory_j} \gamma^j R_\parameter(s_j, a_j) p_\parameter(\trajectory_j, \policy) d \trajectory_j \nonumber
\end{align}
Equation~\ref{eqn:policy-gradient-full-expectation} to Equation~\ref{eqn:policy-gradient-partial-expectation} uses the fact that we only need to sample up to time step $j$ in order to compute $\gamma^j R_\parameter(s_j, a_j)$. Everything beyond time step $j$ does not affect the expectation up to time step $j$.
We can compute the policy gradient by:
\begin{align}
    &\nabla_\policy J_\parameter(\policy) \nonumber \\ &= \nabla_\policy \sum\nolimits_{j=1}^h \int\nolimits_{\trajectory_j} \gamma^j R_\parameter(s_j, a_j) p_\parameter(\trajectory_j, \policy) d \trajectory_j \nonumber \\
    &= \sum\nolimits_{j=1}^h \int\nolimits_{\trajectory_j} \gamma^j R_\parameter(s_j, a_j) \nabla_\policy p_\parameter(\trajectory_j, \policy) d \trajectory_j \label{eqn:policy-gradient-passing-gradient} \\
    &= \sum\nolimits_{j=1}^h \int\nolimits_{\trajectory_j} \gamma^j R_\parameter(s_j, a_j) p_\parameter(\trajectory_j, \policy) \nabla_\policy \log p_\parameter(\trajectory_j, \policy) d \trajectory_j \label{eqn:policy-gradient-log-trick}
\end{align}
where Equation~\ref{eqn:policy-gradient-passing-gradient} is because only the probability term is dependent on policy $\policy$, and  Equation~\ref{eqn:policy-gradient-log-trick} is by $\nabla_\policy p_\parameter = p_\parameter \nabla_\policy \log p_\parameter$.

We can now merge the integral back to an expectation over trajectory $\trajectory_j$ by merging the probability term $p_\parameter$ and the integral:
\begin{align}
    \nabla_\policy J_\parameter(\policy)
    &= \sum\nolimits_{j=1}^h \mathop{\mathbf{E}}\nolimits_{\trajectory_j \sim p_\parameter(\trajectory_j, \policy)} \left[ \gamma^j R_\parameter(s_j, a_j) \nabla_\policy \log p_\parameter(\trajectory_j, \policy) \right] \nonumber \\
    &= \sum\nolimits_{j=1}^h \mathop{\mathbf{E}}\nolimits_{\trajectory \sim p_\parameter(\trajectory, \policy)} \left[ \gamma^j R_\parameter(s_j, a_j) \nabla_\policy \log p_\parameter(\trajectory_j, \policy) \right] \nonumber \\
    &= \mathop{\mathbf{E}}\nolimits_{\trajectory \sim p_\parameter(\trajectory, \policy)} \left[ \sum\nolimits_{j=1}^h \gamma^j R_\parameter(s_j, a_j) \nabla_\policy \log p_\parameter(\trajectory_j, \policy) \right] \nonumber \\
    &= \mathop{\mathbf{E}}\nolimits_{\trajectory \sim p_\parameter(\trajectory, \policy)} \left[ \sum\nolimits_{j=1}^h \gamma^j R_\parameter(s_j, a_j) \nabla_\policy \left( \sum\nolimits_{i=1}^j \log \policy(a_i \mid s_i) + \sum\nolimits_{i=1}^j \log p_\parameter(s_i, a_i, s_{i+1}) \right) \right] \label{eqn:policy-gradient-probability-expansion} \\
    &=  \mathop{\mathbf{E}}\nolimits_{\trajectory \sim p_\parameter(\trajectory, \policy)} \left[ \sum\nolimits_{j=1}^h \gamma^j R_\parameter(s_j, a_j) \sum\nolimits_{i=1}^j \nabla_\policy  \log \policy(a_i \mid s_i) \right] \nonumber \\
    &= \mathop{\mathbf{E}}\nolimits_{\trajectory \sim p_\parameter(\trajectory, \policy)} \left[ \sum\nolimits_{j=1}^h \sum\nolimits_{i=1}^j \gamma^j R_\parameter(s_j, a_j) \nabla_\policy  \log \policy(a_i \mid s_i) \right] \nonumber \\
    &= \mathop{\mathbf{E}}\nolimits_{\trajectory \sim p_\parameter(\trajectory, \policy)} \left[ \sum\nolimits_{i=1}^h \sum\nolimits_{j=i}^h \gamma^j R_\parameter(s_j, a_j) \nabla_\policy  \log \policy(a_i \mid s_i) \right] \nonumber \\
    &=  \mathop{\mathbf{E}}\nolimits_{\trajectory \sim p_\parameter(\trajectory, \policy)} \left[ \nabla_\policy \Phi_\parameter(\trajectory, \policy) \right] \label{eqn:policy-gradient-final-phi}
\end{align}
where Equation~\ref{eqn:policy-gradient-probability-expansion} is by expanding the probability of seeing trajectory $\trajectory_j$ when parameter $\parameter$ and policy $\policy$ are used, where the probability decomposes into the first term action probability $\policy(a_i \mid s_i)$, and the second term transition probability $p_\parameter(s_i, a_i, s_{i+1})$, which is independent of policy $\policy$ and thus disappears.
The last equation in Equation~\ref{eqn:policy-gradient-final-phi} connects back to the definition of $\Phi$ as defined in the statement of Theorem~\ref{thm:policy-gradient-unbiased}.
$\Phi$ is easy to compute and easy to differentiate through. We can therefore sample a set of trajectories $\{\trajectory\}$ to compute the corresponding $\Phi$ and its derivative to get the unbiased policy gradient estimate.
\end{proof}

\begin{proof}[Second part of the proof (second-order derivatives)]
Given the policy gradient theorem as we recall in the above derivation, we have:
\begin{align}
    \nabla_\policy J_\parameter(\policy) = \mathop{\mathbf{E}}\nolimits_{\trajectory \sim p_\parameter(\trajectory, \policy)} \left[ \nabla_\policy \Phi_\parameter(\trajectory, \policy) \right] \label{eqn:policy-gradient-first-order-form}
\end{align}
We can compute the derivative of Equation~\ref{eqn:policy-gradient-first-order-form} by:
\begin{align}
    \nabla^2_\policy J_\parameter(\policy) &= \nabla_\policy \nabla_\policy J_\parameter(\policy) \nonumber \\
    &= \nabla_\policy \mathop{\mathbf{E}}\nolimits_{\trajectory \sim p_\parameter(\trajectory, \policy)} \left[ \nabla_\policy \Phi_\parameter(\trajectory, \policy) \right] \nonumber \\
    &= \nabla_\policy \int_{\trajectory} \nabla_\policy \Phi_\parameter(\trajectory, \policy) p_\parameter(\trajectory, \policy) d \trajectory \nonumber \\
    &= \int_{\trajectory} \left[ \nabla_\policy \Phi_\parameter(\trajectory, \policy) \nabla_\policy p_\parameter(\trajectory, \policy)^\top + \nabla^2_\policy \Phi_\parameter(\trajectory, \policy) p_\parameter(\trajectory, \policy) \right] d \trajectory  \label{eqn:policy-gradient-chain-rule} \\
    &= \int_{\trajectory} \left[ \nabla_\policy \Phi_\parameter(\trajectory, \policy) \nabla_\policy \log p_\parameter(\trajectory, \policy)^\top + \nabla^2_\policy \Phi_\parameter(\trajectory, \policy) \right] p_\parameter(\trajectory, \policy) d \trajectory \nonumber \\
    &= \mathop{\mathbf{E}}\nolimits_{\trajectory \sim p_\parameter(\trajectory, \policy)} \left[ \nabla_\policy \Phi_\parameter(\trajectory, \policy) \nabla_\policy \log p_\parameter(\trajectory, \policy)^\top + \nabla^2_\policy \Phi_\parameter(\trajectory, \policy) \right] \label{eqn:policy-gradient-policy-hessian}
\end{align}
where Equation~\ref{eqn:policy-gradient-chain-rule} passes gradient inside the integral and applies chain rule. Equation~\ref{eqn:policy-gradient-policy-hessian} provides an unbiased estimate of the second-order derivative $\nabla^2_\policy J_\parameter(\policy)$.

Similarly, we can compute:
\begin{align}
    \nabla^2_{\parameter \policy} J_\parameter(\policy) &= \nabla_\parameter \nabla_\policy J_\parameter(\policy) \nonumber \\
    &= \nabla_\parameter \mathop{\mathbf{E}}\nolimits_{\trajectory \sim p_\parameter(\trajectory, \policy)} \left[ \nabla_\policy \Phi_\parameter(\trajectory, \policy) \right] \nonumber \\
    &= \nabla_\parameter \int_{\trajectory} \nabla_\policy \Phi_\parameter(\trajectory, \policy) p_\parameter(\trajectory, \policy) d \trajectory \nonumber \\
    &= \int_{\trajectory} \left[ \nabla_\policy \Phi_\parameter(\trajectory, \policy) \nabla_\parameter p_\parameter(\trajectory, \policy)^\top + \nabla^2_{\parameter \policy} \Phi_\parameter(\trajectory, \policy) p_\parameter(\trajectory, \policy) \right] d \trajectory \nonumber \\
    &= \int_{\trajectory} \left[ \nabla_\policy \Phi_\parameter(\trajectory, \policy) \nabla_\parameter \log p_\parameter(\trajectory, \policy)^\top + \nabla^2_{\parameter \policy} \Phi_\parameter(\trajectory, \policy) \right] p_\parameter(\trajectory, \policy) d \trajectory \nonumber \\
    &= \mathop{\mathbf{E}}\nolimits_{\trajectory \sim p_\parameter(\trajectory, \policy)} \left[ \nabla_\policy \Phi_\parameter(\trajectory, \policy) \nabla_\parameter \log p_\parameter(\trajectory, \policy)^\top + \nabla^2_{\parameter \policy} \Phi_\parameter(\trajectory, \policy) \right] \label{eqn:policy-gradient-second-order}
\end{align}

Equation~\ref{eqn:policy-gradient-policy-hessian} and Equation~\ref{eqn:policy-gradient-second-order} both serve as unbiased estimates of the corresponding second-order derivatives. We can sample a set of trajectories to compute both of them and get an unbiased estimate of the second-order derivatives. This concludes the proof of Theorem~\ref{thm:policy-gradient-unbiased}.
\end{proof}

\bellmanUnbiased*
\begin{proof}[First part of the proof (first-order derivative)]
By the definition of $J_\parameter(\policy) = \frac{1}{2} \mathop{\mathbf{E}}\nolimits_{\trajectory \sim \policy, \parameter} \left[ \delta^2(\trajectory, \policy) \right]$, we can compute its first-order derivative by:
\begin{align}
    \nabla_\policy J_\parameter(\policy) &= \nabla_\policy \frac{1}{2} \mathop{\mathbf{E}}\nolimits_{\trajectory \sim \policy, \parameter} \left[ \delta_\parameter^2(\trajectory, \policy) \right] \nonumber \\
    &= \nabla_\policy \frac{1}{2} \int_{\trajectory} \delta_\parameter^2(\trajectory, \policy) p_\parameter(\trajectory, \policy) d \trajectory \nonumber \\
    &= \int_{\trajectory}  \left[ p_\parameter(\trajectory, \policy) \delta_\parameter(\trajectory, \policy) \nabla_\policy \delta_\parameter(\trajectory, \policy) + \frac{1}{2} \delta_\parameter^2(\trajectory, \policy) \nabla_\policy p_\parameter(\trajectory, \policy) \right] d \trajectory \nonumber \\
    &= \int_{\trajectory}  \left[ \delta_\parameter(\trajectory, \policy) \nabla_\policy \delta_\parameter(\trajectory, \policy) + \frac{1}{2} \delta_\parameter^2(\trajectory, \policy) \nabla_\policy \log p_\parameter(\trajectory, \policy) \right] p_\parameter(\trajectory, \policy) d \trajectory \nonumber \\
    &= \mathop{\mathbf{E}}\nolimits_{\trajectory \sim \policy, \parameter} \left[ \delta_\parameter(\trajectory, \policy) \nabla_\policy \delta_\parameter(\trajectory, \policy) + \frac{1}{2} \delta_\parameter^2(\trajectory, \policy) \nabla_\policy \log p_\parameter(\trajectory, \policy) \right] \label{eqn:bellman-gradient-first-order}
\end{align}
\end{proof}

\begin{proof}[Second part of the proof (second-order derivative)]
Given Equation~\ref{eqn:bellman-gradient-first-order}, we can further compute the second-order derivatives by:
\begin{align}
    \nabla^2_\policy J_\parameter(\policy) &= \nabla_\policy \nabla_\policy J_\parameter(\policy) \nonumber \\
    &= \nabla_\policy \mathop{\mathbf{E}}\nolimits_{\trajectory \sim \policy, \parameter} \left[ \delta_\parameter(\trajectory, \policy) \nabla_\policy \delta_\parameter(\trajectory, \policy) + \frac{1}{2} \delta_\parameter^2(\trajectory, \policy) \nabla_\policy \log p_\parameter(\trajectory, \policy) \right] \nonumber \\
    &= \nabla_\policy \int_{\trajectory}  \left[ \delta_\parameter(\trajectory, \policy) \nabla_\policy \delta_\parameter(\trajectory, \policy) + \frac{1}{2} \delta_\parameter^2(\trajectory, \policy) \nabla_\policy \log p_\parameter(\trajectory, \policy) \right] p_\parameter(\trajectory, \policy) d \trajectory \nonumber \\
    &= \int_{\trajectory} \left( \nabla_\policy \delta \nabla_\policy \delta^\top + \delta \nabla^2_\policy \delta + \delta \nabla \log p_\parameter \nabla_\policy \delta^\top + \frac{1}{2} \delta^2 \nabla^2 \log p_\parameter \right) p_\parameter \nonumber \\
    & \quad \quad \quad + \left( \delta \nabla_\policy \delta(\trajectory, \policy) + \frac{1}{2} \delta^2 \nabla_\policy \log p_\parameter \right) p_\parameter \nabla \log p_\parameter^\top \quad  d \trajectory \nonumber \\
    &= \mathop{\mathbf{E}}\nolimits_{\trajectory \sim \policy, \parameter} \left[ \nabla_\policy \delta \nabla_\policy \delta^\top + \delta \nabla^2_\policy \delta + \delta \nabla \log p_\parameter \nabla_\policy \delta^\top + \delta \nabla_\policy \delta(\trajectory, \policy)\nabla \log p_\parameter^\top + O(\delta^2) \right] \nonumber \\
    &= \mathop{\mathbf{E}}\nolimits_{\trajectory \sim \policy, \parameter} \left[ \nabla_\policy \delta \nabla_\policy \delta^\top + O(\delta) \right] \nonumber
\end{align}

Similarly, we have:
\begin{align}
    \nabla^2_{\parameter\policy} J_\parameter(\policy) &= \nabla_\parameter \nabla_\policy J_\parameter(\policy) \nonumber \\
    &= \nabla_\parameter \mathop{\mathbf{E}}\nolimits_{\trajectory \sim \policy, \parameter} \left[ \delta_\parameter(\trajectory, \policy) \nabla_\policy \delta_\parameter(\trajectory, \policy) + \frac{1}{2} \delta_\parameter^2(\trajectory, \policy) \nabla_\policy \log p_\parameter(\trajectory, \policy) \right] \nonumber \\
    &= \nabla_\parameter \int_{\trajectory}  \left[ \delta_\parameter \nabla_\policy \delta_\parameter + \frac{1}{2} \delta_\parameter^2 \nabla_\policy \log p_\parameter \right] p_\parameter d \trajectory \nonumber \\
    &= \int_{\trajectory}  \left( \nabla_\policy \delta \nabla_\parameter \delta^\top + \delta \nabla^2_{\parameter \policy} \delta + \delta \nabla_\policy \log p_\parameter \nabla_\parameter \delta^\top + \frac{1}{2} \delta^2 \nabla^2_{\parameter \policy} \log p_\parameter \right) p_\parameter  \nonumber \\
    & \quad \quad \quad + \left( \delta \nabla_\policy \delta + \frac{1}{2} \delta^2 \nabla_\policy \log p_\parameter \right) p_\parameter \nabla_\parameter \log p_\parameter^\top \quad d \trajectory \nonumber \\
    &= \mathop{\mathbf{E}}\nolimits_{\trajectory \sim \policy, \parameter} \left[  \nabla_\policy \delta \nabla_\parameter \delta^\top + \delta \nabla^2_{\parameter \policy} \delta + \delta \nabla_\policy \log p_\parameter \nabla_\parameter \delta^\top + \delta \nabla_\policy \delta \nabla_\parameter \log p_\parameter^\top + O(\delta^2) \right] \nonumber \\
    &= \mathop{\mathbf{E}}\nolimits_{\trajectory \sim \policy, \parameter} \left[  \nabla_\policy \delta \nabla_\parameter \delta^\top + \left( \nabla^2_{\parameter \policy} \delta + \nabla_\policy \log p_\parameter \nabla_\parameter \delta^\top +  \nabla_\policy \delta \nabla_\parameter \log p_\parameter^\top \right) \delta + O(\delta^2) \right]
\end{align}
which concludes the proof.
\end{proof}

\section{Additional Discussions of Decision-focused Learning}\label{sec:challenges}
In this section, we provide additional discussions of applying decision-focused learning to MDPs problems.

\subsection{Smoothness of the Optimal Policy Derived From Reinforcement Learning Solver}\label{sec:soft-rl-solver}
In Equation~\ref{eqn:full-gradient}, we compute the gradient of the final evaluation metric with respect to the predictive model by applying chain rule.
This implicitly requires each individual component in the chain rule to be well-defined. Specifically, the mapping from the MDP parameters to the optimal policy needs to be smooth so that we can compute a meaningful derivative of the policy with respect to the MDP parameters.
However, this smoothness requirement is only required in the training time to make the gradient computation available. Once the training is finished, there is no restriction on the policy and the corresponding solver. This smoothness requirement does not restrict the kind of problems that we can solve. We just need to find a solver that can give a smooth policy to ensure the differentiability at training time, e.g., soft actor critic and soft Q learning.


Specifically, the assumption on smooth policy is similar to the idea of soft Q-learning~\cite{haarnoja2017reinforcement} and soft actor-critic~\cite{haarnoja2018soft} proposed by Haarnoja et al. Soft Q-learning relaxes the Bellman equation to a soft Bellman equation to make the policy smoother, while soft actor-critic adds an entropy term as regularization to make the optimal policy smoother. These relaxed policy not only can make the training smoother as stated in the above papers, but also can allow back-propagation through the optimal policy to the input MDP parameters in our paper.
These benefits are all due to the smoothness of the optimal policy.
Similar issues arise in decision-focused learning in discrete optimization, with Wilder et al.~\cite{wilder2019melding} proposing to relaxing the optimal solution by adding a regularization term, which serves as the same purpose as we relax our optimal policy in the sequential decision problem setting.





\subsection{Unbiased Second-order Derivative Estimates}
As we discuss in Section~\ref{sec:hessian}, correctly approximating the second-order derivatives is the crux of our algorithm. Incorrect approximation may lead to incorrect gradient direction, which can further lead to divergence.
Since the second-order derivative formulation as stated in Theorem~\ref{thm:policy-gradient-unbiased} and Theorem~\ref{thm:bellman-unbiased} are both unbiased derivative estimate. However, their accuracy depends on how many samples we use to approximate the derivatives. In our experiments, we use $100$ sampled trajectories to approximate the second-order derivatives across three domains. The number of samples required to get a sufficiently accurate derivative estimate may depend on the problem size. Larger problems may require more samples to get a good derivative estimate, but more samples also implies more computation cost required to run the back-propagation.

In practice, we find that normalization effect of the Hessian term as discussed in Section~\ref{sec:hessian} is very important to reduce the variance caused by the incorrect derivative estimate. Additionally, we also notice that adding a small additional predictive loss term to run back-propagation can stabilize the training process because the predictive loss does not suffer from sampling variance.
This is why we add a weighted predictive loss to the back-propagation in Algorithm~\ref{alg:decision-focused-learning}.

\subsection{Impact of Optimality in the Forward Pass}
In order to differentiate through the KKT conditions, we need the policy $\policy^*$ return by the reinforcement learning solver to be optimal in Figure~\ref{fig:flowchart}. However, sub-optimal solution is often reached by the reinforcement learning solver and the optimality can impact the gradient computed from differentiating through the KKT conditions.

In this section, we analyze the impact of a sub-optimal policy produced by the reinforcement learning solver.
When the problem is smooth, or more precisely when the function $J_\theta(\pi)$ is smooth around the optimal policy $\pi^*$, we can bound the gradients $\nabla^2_\pi J_\theta(\pi')$ and $\nabla_\pi J_\theta(\pi')$ computed in Equation~\ref{eqn:full-gradient} using a sub-optimal policy $\pi'$ by the gradients computed using the optimal policy $\pi^*$. Specifically, if the Hessian $\nabla^2_\pi J_\theta(\pi^*)$ is sufficiently far from singular, the difference between two gradients computed from sub-optimal and optimal policy using Equation~\ref{eqn:full-gradient} can be written as:
\begin{align}
\left| \frac{d ~ \text{Eval}(\pi')}{d \pi} (\nabla^2_{\pi} J_{\theta}(\pi'))^{-1} \nabla^2_{\theta \pi} J_{\theta}(\pi') \frac{d \theta}{d \weight} 
    - \frac{d ~ \text{Eval}(\pi^*)}{d \pi} (\nabla^2_{\pi} J_{\theta}(\pi^*))^{-1} \nabla^2_{\theta \pi} J_{\theta}(\pi^*) \frac{d \theta}{d \weight} \right| \nonumber    
\end{align}
which can be further bounded by applying telescoping sum to decompose the difference into linear combination of the difference in each individual gradient term.
This suggests that when the smoothness condition of the above derivatives is met, we can bound the error incurred by sub-optimal policy.




\section{Experimental Setup}\label{sec:experiment-setup}

In this section, we describe how we randomly generate the MDP problems and the corresponding missing parameters.

\paragraph{Feature generation}
Across all three domains, once the missing parameters are generated, we feed each MDP parameter into a randomly initialized neural network with two intermediate layers each with $64$ neurons, and an output dimension size $16$ to generate a feature vector of size $16$ for the corresponding MDP parameter.
For example, in the gridworld example, each grid cell comes with a missing reward. So the feature corresponding to this grid cell and the missing reward is generated by feeding the missing reward into a randomly initialized neural network to generate a feature vector of size $16$ for this particular grid cell.
We repeat the same process for all the parameters in the MDP problem, e.g., all the grid cells in the gridworld problem.
The randomly initiated neural network uses ReLU layers as nonlinearity followed by a linear layer in the end.
The generated features are normalized to mean $0$ and variance $1$, and we add Gaussian noise $\mathcal{N}(0,1)$ to the features, with a signal noise ratio is $1:3$, to model that the original missing parameters may not be perfectly recovered from the noisy features.
The predictive model we use to map from generated noisy features to the missing parameters is a single layer neural network with $16$ neurons.

\paragraph{Training parameters}
Across all three examples, we consider the discounted setting where the discount factor is $\gamma = 0.95$.
The learning rate is set to be $\alpha = 0.01$.
The number of demonstrated trajectories is set to be $100$ in both the random and near-optimal settings.

\paragraph{Reinforcement learning solvers}
In order to train the optimal policy, in the gridworld example, we use tabular value-iteration algorithm to learn the Q value of each state action pair.
In the snare finding and the TB problems, since the state space is continuous, we apply DDQN~\cite{stable-baselines3,van2016deep} to train the Q function and the corresponding policy, where we use a neural network with two intermediate layers each with $64$ neurons to represent the function approximators of the Q values.
There is one exception in the runtime plot of the snare finding problem in Figure~\ref{fig:snare-runtime}, where the full Hessian computation is infeasible when a two layer neural network is used. Thus we use an one layer neural network with $64$ neurons only to test the runtime of different Hessian approximations.

\subsection{Gridworld Example With Missing Rewards}\label{sec:appendix-gridworld}
\paragraph{Problem setup}
We consider a $5 \times 5$ Gridworld environment with unknown rewards as our MDP problems with unknown parameters. The bottom left corner is the starting point and the top right corner is a safe state with a high reward drawn from a normal distribution $\mathcal{N}(5,1)$.
The agent can walk between grid cells by going north, south, east, west, or deciding to stay in the current grid cells.
So the number of available actions is $5$, while the number of available states is $5 \times 5 = 25$.

The agent collects reward when the agent steps on each individual grid cell. There is $20\%$ chance that each intermediate grid cell is a cliff that gives a high penalty drawn from another normal distribution $\mathcal{N}(-10, 1)$. All the other $80\%$ of grid cells give rewards drawn from $\mathcal{N}(0,1)$. The goal of the agent is to collect as much reward as possible. We consider a fixed time horizon case with $20$ steps, which is sufficient for the agent to go from bottom left to the top right corner.

\paragraph{Training details}
Within each individual training step for each MDP problem with missing parameters, we first predict the rewards using the predictive model, and then solve the resulting problem using tabular value-iteration. We run in total $10000$ iterations to learn the Q values, which are later used to construct the optimal policy.
To relax the optimal policy given by the RL solver, we relax the Bellman equation used to run value-iteration by relaxing all the argmax and max operators in the Bellman equation to softmax with temperature $0.1$, i.e., we use  $\text{SOFTMAX}(0.1 \cdot \text{Q-values})$ to replace all the argmax over Q values. The choice of the tempreratue $0.1$ is to make sure that the optimal policy is smooth enough but the relaxation does not impact the optimal policy too much as well.

\paragraph{Random and near-optimal trajectories generation}
To generate the random trajectories, we have the agent randomly select actions between all actions.
To generate the near-optimal trajectories, we replace the softmax with temperature $0.1$ by softmax with temperature $1$ and train an RL agent using ground truth reward values by $50000$ value-iterations to get a near-optimal policy.
We then use the trained near-optimal policy to generate $100$ independent trajectories as our near-optimal demonstrated trajectories.



\subsection{Snare Finding Problem With Missing Transition Probability} \label{sec:appendix-snare}
\paragraph{Problem setup}
In the snare finding problem, we consider a set of $20$ sites that are vulnerable to poaching activity. We randomly select $20\%$ of the sites as high-risk locations where the probability of having a poacher coming and placing a snare is randomly drawn from a normal distribution $\mathcal{N}(0.8, 0.1)$, while the remaining $80\%$ of low-risk sites with probability $\mathcal{N}(0.1, 0.05)$ having a poacher coming to place a snare.
These transition probabilities are not known to the ranger, and the ranger has to rely on features of each individual site to predict the corresponding missing transition probability.

We assume the maximum number of snare is $1$ per location, meaning that if there is a snare and it has not been removed by the ranger, then the poacher will not place an additional snare there until the snare is removed.
The ranger only observes a snare when it is removed. 
Thus the MDP problem with given parameters is partially observable, where the state maintained by the ranger is the belief of whether a site contains a snare or not, which is a fractional value between $0$ and $1$ for each site.

The available actions for the ranger are to select a site from $20$ sites to visit. If there is a snare in the location, the ranger successfully removes the snare and gets reward $1$ with probability $0.9$, and otherwise the snare remains there with a reward $-1$. If there is no snare in the visited site, the ranger gets reward $-1$.
Thus the number of actions to the ranger is $20$, while the state space is continuous since the ranger uses continuous belief as the state.

\paragraph{Training details}
To solve the optimal policy from the predicted parameters, we run DDQN with $1000$ iterations to collect random experience and  $10000$ iterations to train the model. We use a replay buffer to store all the past experience that the agent executed before.
To soften the optimal policy, we also use a relaxed Bellman equation as stated in Section~\ref{sec:appendix-snare}.
Because the cumulative reward and the corresponding Q values in this domain is relatively smaller than the Gridworld domain, we replace all the argmax and max operators by softmax with a larger temperature $1$ to reflect the relatively smaller reward values.

\paragraph{Random and near-optimal trajectories generation}
Similar to Section~\ref{sec:appendix-gridworld}, we generate the random trajectories by having the agent choose action from all available actions uniformly at random.
To generate the near-optimal trajectories, we replace all the softmax with temperature $1$ by softmax with temperature $5$, and we use the ground truth transition probabilities to train the RL agent by DDQN using $50000$ iterations to generate a near-optimal policy.
The near-optimal trajectories are then generated by running the trained near-optimal policy for $100$ independent runs.

\subsection{Tuberculosis With Missing Transition Probability}
\paragraph{Problem setup}
There are a total of 5 patients who need to take their medication at each time-step for 30 time-steps. For each patient, there are 2 states -- non-adhering (0), and adhering (1). The patients are assumed to start from a non-adhering state. Then, in subsequent time-steps, the patients' states evolve based on their current state and whether they were intervened on by a healthcare worker according to a transition matrix.

The raw transition probabilities for different patients are taken from~\cite{killian2019learning}.\footnote{The raw transition probabilities taken from~\cite{killian2019learning} are only used to generate synthetic patients.
}
However, these raw probabilities do not record a patient's responsiveness to an intervention. To incorporate the effect of intervening, we sample numbers from $U(0, 0.4)$, and (a) add them to the probability of adhering when intervened on, and (b) subtract them from the probability of adhering when not. Finally, we clip the probabilities to the range of $[0.05, 0.95]$ and re-normalize. We use the raw transition probabilities and the randomly generated intervention effect to model the behavior of our synthetic patients and generate all the training trajectories accordingly. The entire transition matrix for each patient is then fed as an input to the feature generation network to generate features for that patient. In this example, we assume the transition matrices to be missing parameters, and try to learn a predictive model to recover the transition matrices from the generated features using either two-stage or various decision-focused learning methods as discussed in the main paper.

Given the synthetic patients, we consider a healthcare worker who has to choose one patient at every time-step to intervene on. However, the healthcare worker can only observe the `true state' of a patient when she intervenes on them. At every other time, she has a `belief' of the patient's state that is constructed from the most recent observation and the predicted transition probabilities. Therefore, the healthcare worker has to learn a policy that maps from these belief states to the action of whom to intervene on, such that the sum of adherences of all patients is maximised over time. The healthcare worker gets a reward of $1$ for an adhering patient and $0$ for a non-adhering one. Like Problem \ref{sec:appendix-snare}, this problem has a continuous state space (because of the belief states) and discrete action space.

\paragraph{Training details}
Same as Section~\ref{sec:appendix-snare}.

\paragraph{Random and near-optimal trajectories generation}
Similar to Section~\ref{sec:appendix-snare}, we generate the random trajectories by having the agent choose action from all available actions uniformly at random.
To generate the near-optimal trajectories, we replace all the softmax with temperature $5$ by softmax with temperature $20$,\footnote{The reason that we use a relatively larger temperature is because the range of the cumulative reward in TB domain is smaller than the previous two domains. In TB domain, the patients could change from non-adhering back to adhering even if there is no intervention, while in contrast, a snare placed in a certain location will not be removed until the ranger visits the place. In other words, the improvement that intervention can introduce is relatively limited compared to the snare finding domain. Thus even though the cumulative reward in TB domain is larger than the previous two domains, the range is smaller and thus we need a larger temperature.} and we use the ground truth transition probabilities to train the RL agent by DDQN using $100,000$ iterations to generate a near-optimal policy.
The near-optimal trajectories are then generated by running the trained near-optimal policy for $100$ independent runs.

\section{Additional Experiment Results}

\begin{figure*}
\centering
\begin{minipage}{\textwidth}
    \centering
    \begin{subfigure}{\linewidth}
        \includegraphics[width=\textwidth]{figures/legend_short.pdf}
    \end{subfigure}
    \begin{subfigure}{0.45\linewidth}
        \centering
        \includegraphics[width=\textwidth]{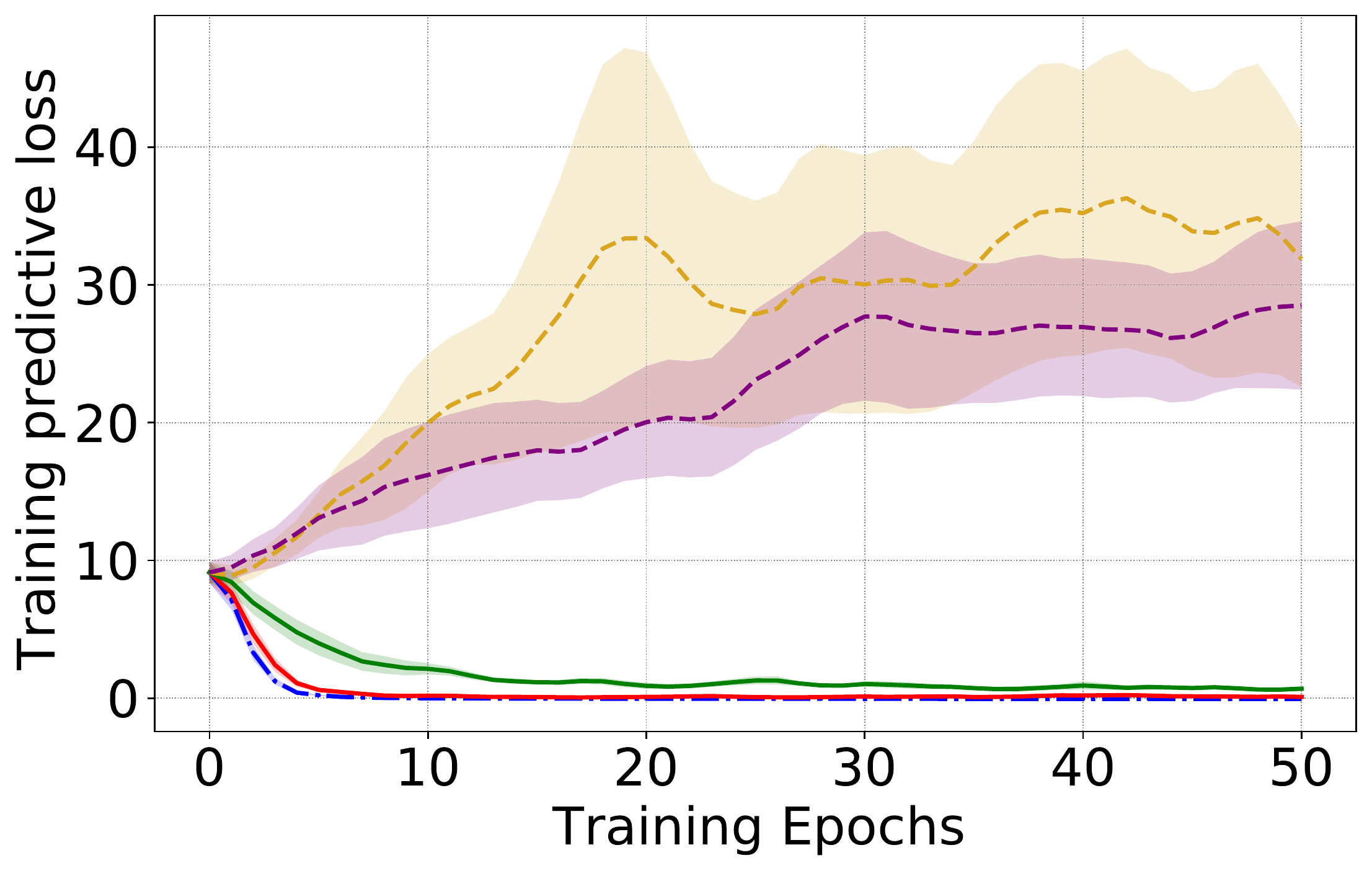}
        \caption{Predictive loss on training MDPs}
        \label{fig:tb-training-loss}
    \end{subfigure}
    \hfill
    \begin{subfigure}{0.45\linewidth}
        \centering
        \includegraphics[width=\textwidth]{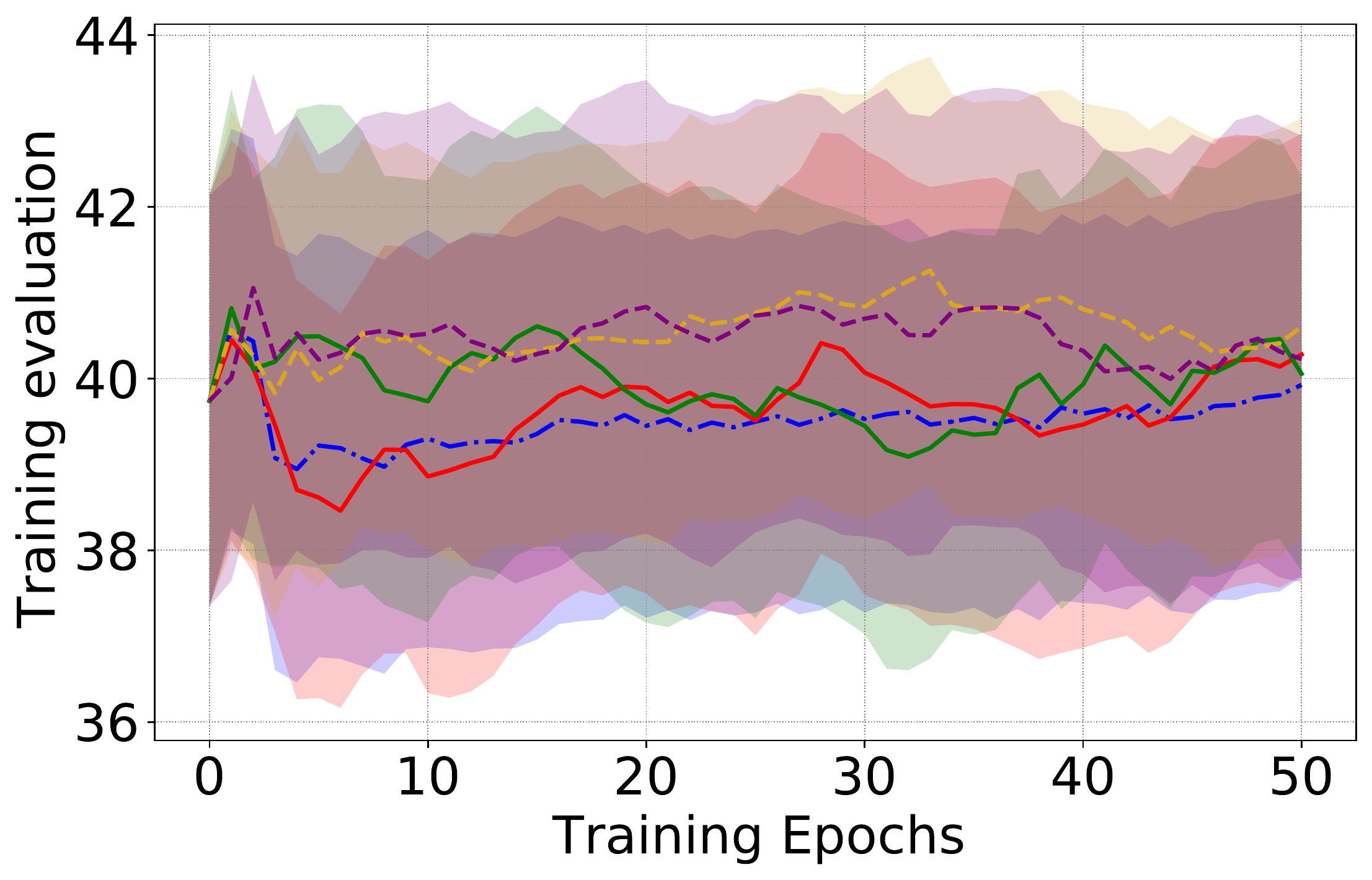}
        \caption{Performance on training MDPs}
        \label{fig:tb-training-eval}
    \end{subfigure}
\end{minipage}
    \caption{Learning curves of for the TB problem with random trajectories.}
    \label{fig:tb}
\end{figure*}

\paragraph{Tuberculosis Adherence}
The results for this problem can be found in Table \ref{table:ope-trajectories}, and the training curves can be found in Figure \ref{fig:tb}. While the standard errors associated with the results \textit{seem} very large, this is in large part because of the way in which we report them. To keep it consistent with other problems, we average the absolute OPE scores for each method across multiple problem instances. However, in the TB case, each problem instance can be very different because the patients in each of these instances are sampled from the transition probabilities previously studied in~\cite{killian2019learning} that have diverse patient behaviour. As a result, the baseline OPE values vary widely across different problem instances, causing a larger variation in Figure~\ref{fig:tb-training-eval} and contributing as the major source of standard deviation.

\begin{figure*}
    \begin{subfigure}{0.45\linewidth}
        \centering
        \includegraphics[width=\textwidth]{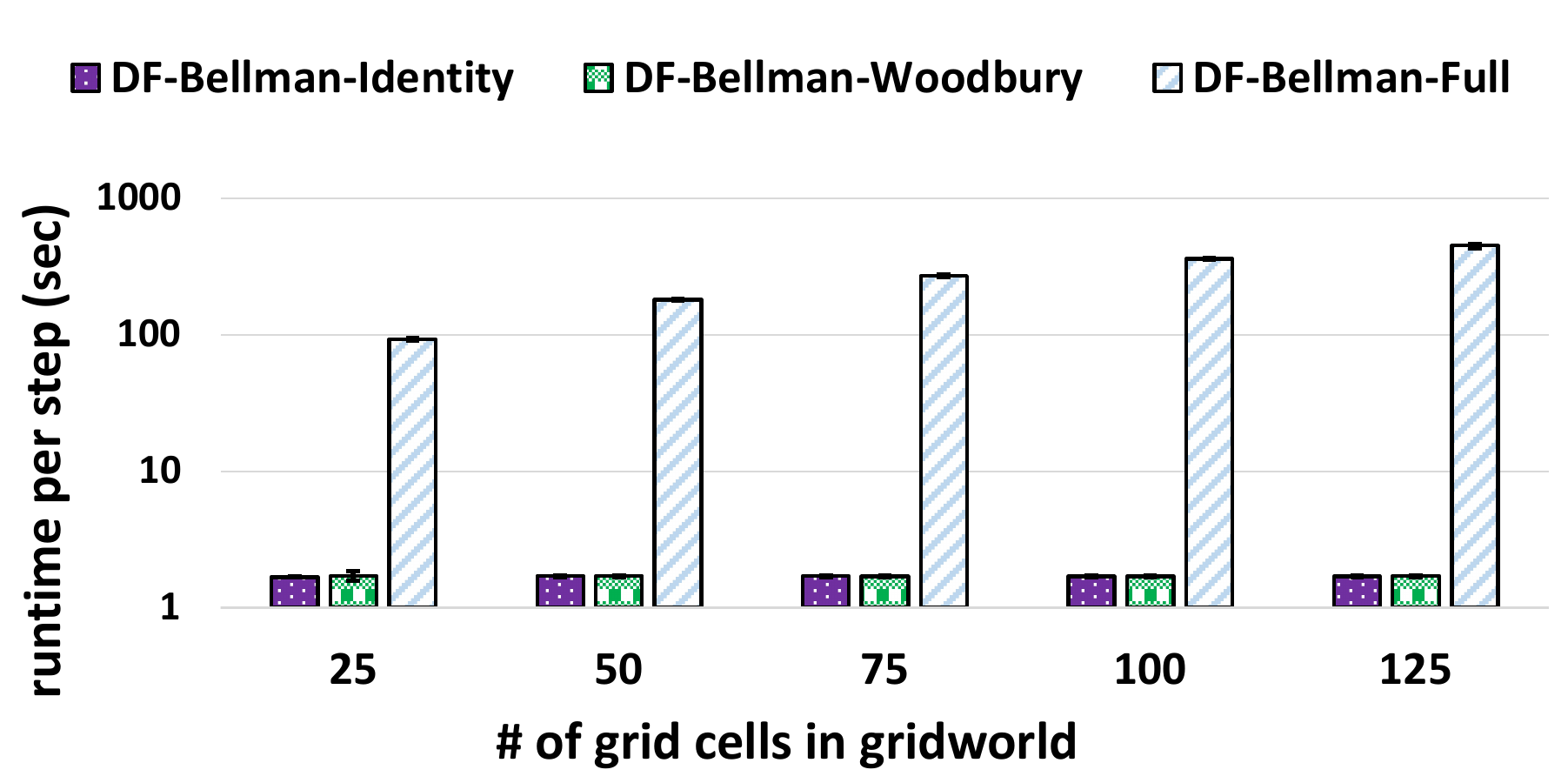}
        \caption{Backpropagation runtime per gradient step of Bellman equation-based decision-focused learning using different Hessian approximations in the gridworld problem.}
        \label{fig:gridworld-runtime-bellman}
    \end{subfigure}
    \hfill
    \begin{subfigure}{0.45\linewidth}
        \centering
        \includegraphics[width=\textwidth]{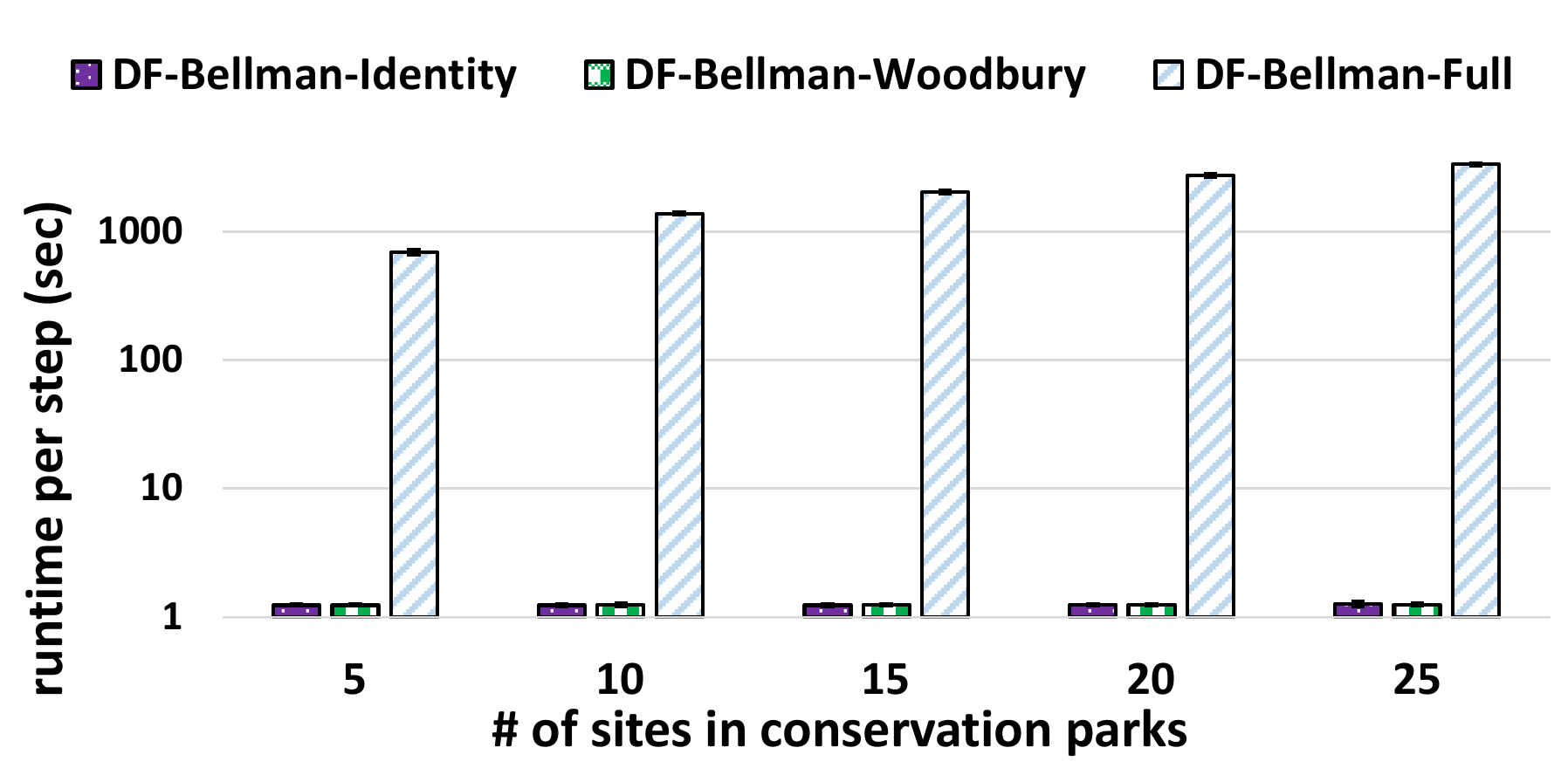}
        \caption{Backpropagation runtime per gradient step of Bellman equation-based decision-focused learning using different Hessian approximations in the snare finding problem.}
        \label{fig:snare-runtime-bellman}
    \end{subfigure}
    \caption{We compare the backpropagation runtime of decision-focused learning methods using bellman optimality with different Hessian approximations. We can see that the runtime of both identity and Woodbury methods largely outperform the runtime of full Hessian computation, demonstrating the importance of the Hessian approximation. Additionally, the runtime of Woodbury method using low-rank approximation is similar to the runtime of identity method. Woodbury method provides a more accurate approximation with a similar runtime.}
    \label{fig:runtime-bellman}
\end{figure*}
\paragraph{Computation cost of Bellman equation-based decision-focused methods} We additionally compare the runtime of the operation of backpropagation per gradient step of Bellman equation-based decision-focused learning using different Hessian approximations. This is the runtime required to compute the gradient in the backward pass. We can see that the runtime of methods using identity and Woodbury methods are much smaller than the runtime of full Hessian approximation. This matches to our analysis in Section~\ref{sec:hessian} and the experimental results in Figure~\ref{fig:gridworld-runtime} and Figure~\ref{fig:snare-runtime}. 

\paragraph{Choice of regularizer $\lambda$ in Algorithm~\ref{alg:decision-focused-learning} and ablation study}

\begin{figure}
    \centering
    \begin{subfigure}{\linewidth}
        \centering
        \includegraphics[width=0.8\textwidth]{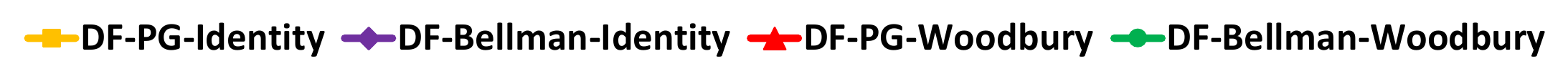}
    \end{subfigure}
    \begin{subfigure}{0.49\linewidth}
        \centering
        \includegraphics[width=\textwidth]{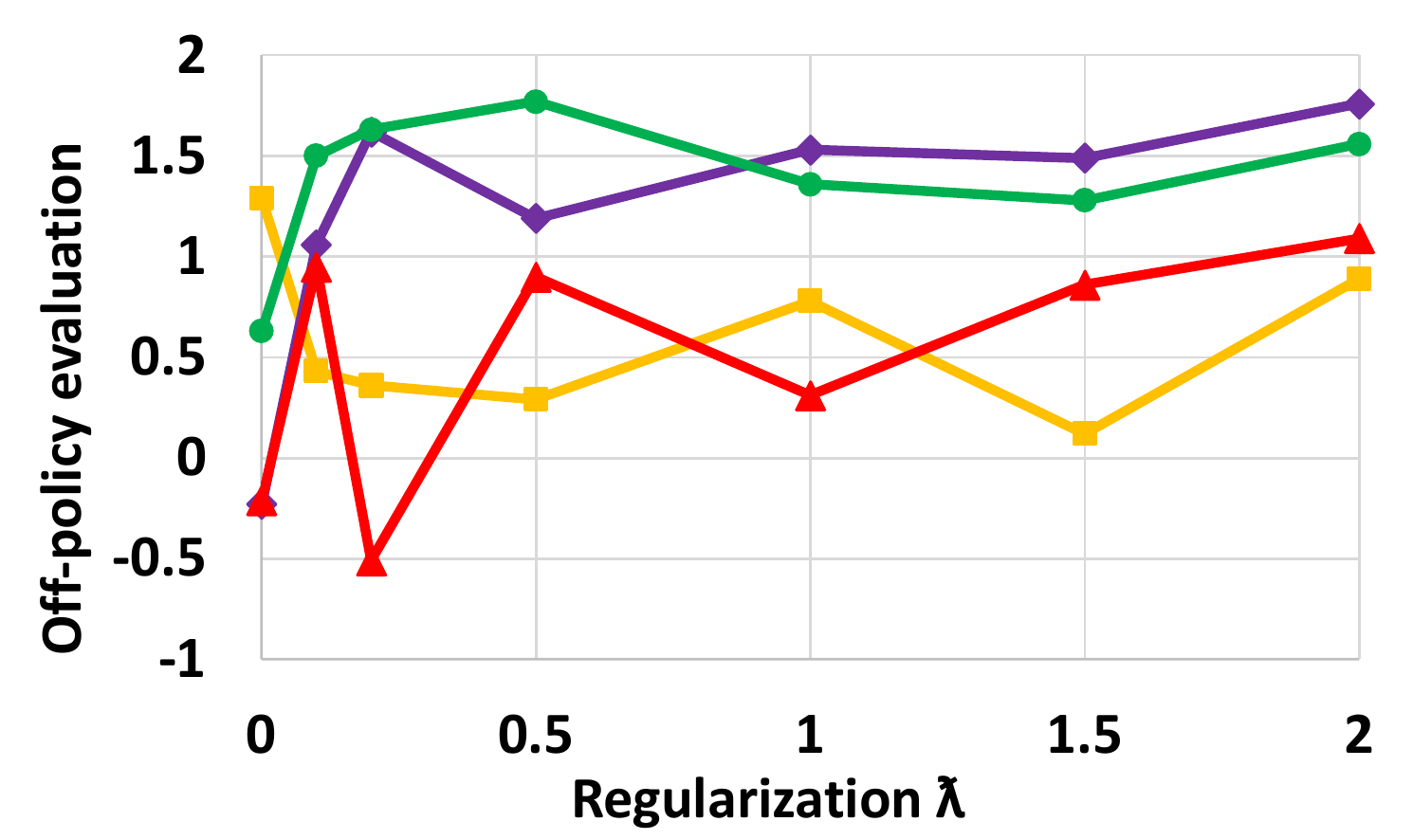}
        \caption{The off-policy evaluation performance of different regularization $\lambda$ in snare-finding problem with random trajectories.}
        \label{fig:ablation-random}
    \end{subfigure}
    \hfill
    \begin{subfigure}{0.49\linewidth}
        \centering
        \includegraphics[width=\textwidth]{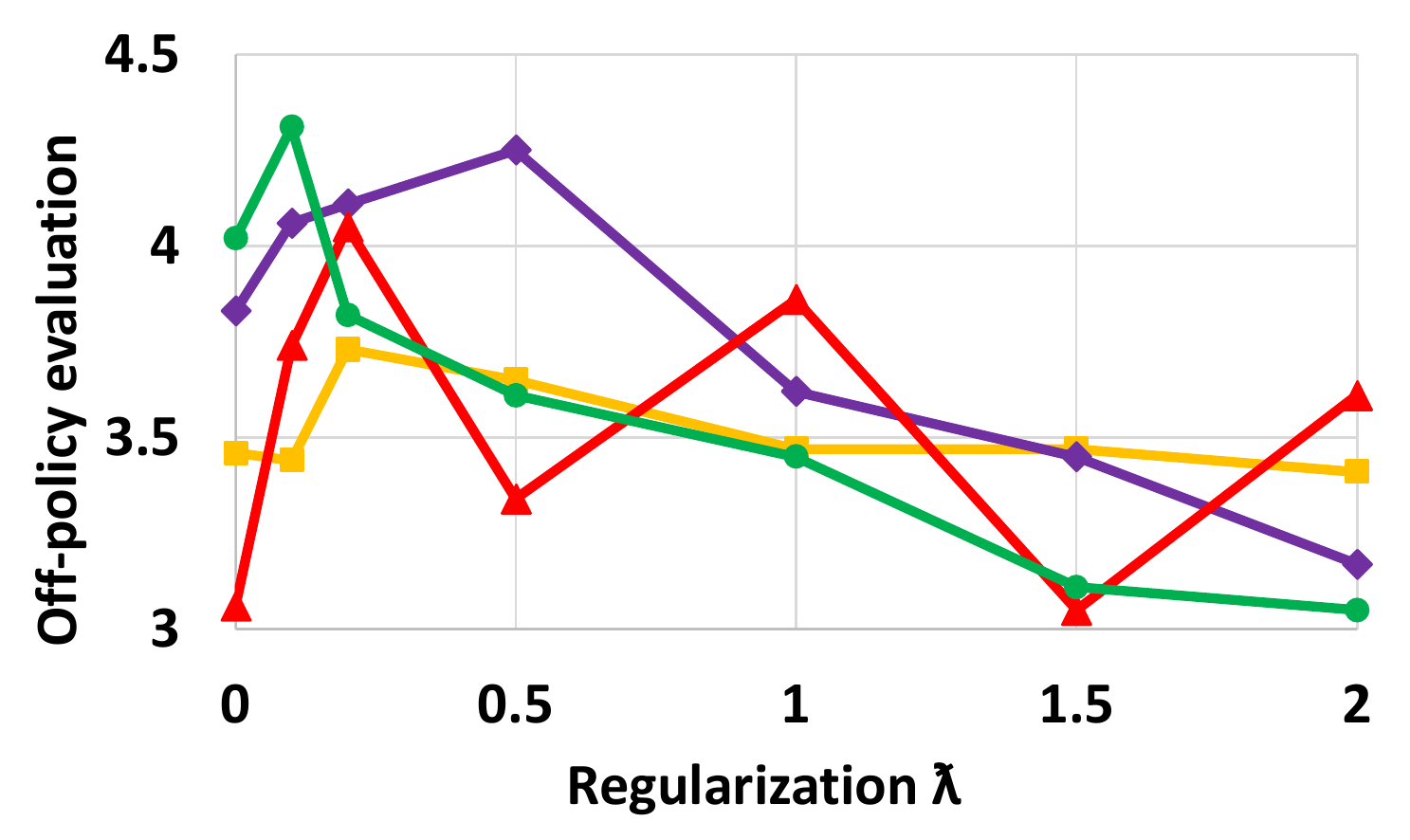}
        \caption{The off-policy evaluation performance of different regularization $\lambda$ in snare-finding problem with near-optimal trajectories.}
        \label{fig:ablation-near-optimal}
    \end{subfigure}
    \caption{Ablation study of different regularization $\lambda$ in Algorithm~\ref{alg:decision-focused-learning} on the snare-finding problem using different decision-focused learning methods.}
    \label{fig:ablation}
\end{figure}

We ran an ablation study by varying the regularization constant $\lambda$ in Algorithm~\ref{alg:decision-focused-learning} using the snare-finding problem. The experimental result is shown in Figure~\ref{fig:ablation}.
The role of regularization in Algorithm~\ref{alg:decision-focused-learning} is to help resolve the non-convexity issue of the off-policy evaluation (OPE) objective. Decision-focused learning methods can easily get trapped by various local minima due to the non-convexity of the OPE metric. Adding a small two-stage loss can improve the convexity of the optimizing objective and thus help improve the performance as well.
We can see that adding small amount of regularization can usually help improve the overall performance in both cases with random and near-optimal trajectories. However, adding too much regularization in Algorithm~\ref{alg:decision-focused-learning} can push decision-focused learning toward two-stage approach, which can degrade the performance sometimes. The right amount of regularization is critical to balance between the issue of convexity and the optimizing objective.

\section{Computing Infrastructure}\label{sec:computing-infrastructure}
All experiments except were run on a computing cluster, where each node is configured with 2 Intel Xeon
Cascade Lake CPUs, 184 GB of RAM, and 70 GB of local scratch space.
Within each experiment, we did not implement parallelization nor use GPU, so each experiment was purely run on a single CPU core.

\end{document}